\PassOptionsToPackage{table}{xcolor}
\documentclass[11pt]{article}
\usepackage[preprint]{acl}

\usepackage{times}
\usepackage{latexsym}
\usepackage[T1]{fontenc}
\usepackage[utf8]{inputenc}
\usepackage{microtype}
\usepackage{inconsolata}

\usepackage{amsmath,amsfonts,mathtools,amssymb,amsthm}
\usepackage{mathrsfs}

\newcommand{\bm}[1]{\boldsymbol{#1}}

\usepackage{booktabs}
\usepackage{tabularx}
\usepackage{array}
\usepackage{multirow}
\usepackage{makecell}
\usepackage{threeparttable}
\usepackage{arydshln}
\usepackage{xcolor}
\usepackage{colortbl}

\usepackage{graphicx}
\usepackage{wrapfig}
\usepackage{subcaption}
\usepackage{adjustbox}
\usepackage{float}

\usepackage{algorithm}
\usepackage{algpseudocode}
\algrenewcommand\textproc{\text}

\usepackage{tikz}
\usepackage{tcolorbox}

\usepackage{ragged2e}
\usepackage{enumitem}
\usepackage{balance}
\usepackage{pifont}
\usepackage{CJKutf8}
\usepackage{hyperref}

\newtheorem{theorem}{Theorem}[section]

\definecolor{mygreen}{HTML}{00B050}
\definecolor{myorange}{HTML}{ED7D31}
\definecolor{rowgray}{gray}{0.97}
\definecolor{avgblue}{RGB}{210,230,250}
\definecolor{headergray}{RGB}{160,160,160}
\definecolor{uc_color}{rgb}{0.99,0.24,0.63}
\definecolor{hc_color}{rgb}{0.02,0.51,0.51}
\definecolor{tc_color}{rgb}{0.99,0.55,0.09}
\definecolor{color1}{cmyk}{0.216,0.176,0,0}
\definecolor{color2}{cmyk}{0.059,0.235,0.392,0}

\newcommand{\hlg}[1]{\colorbox{green!15}{\strut #1}}
\newcommand{\hlp}[1]{\colorbox{purple!10}{\strut #1}}

\renewcommand{\arraystretch}{1.1}

\tikzstyle{mybox} = [draw=black, very thick, rectangle, rounded corners, inner sep=10pt, inner ysep=13pt]
\tikzstyle{fancytitle} = [fill=black, text=white]

\title{\textit{Fishing Out Free Riders}: Shapley-Based Reward Attribution for Parallel Reasoning via Reinforcement Learning}

\author{
  \textbf{Wentao Zhang\textsuperscript{1}\thanks{Equal contribution.}}, 
  \textbf{Haoyu Zhang\textsuperscript{2}\footnotemark[1]}, 
  \textbf{Xinke Jiang\textsuperscript{3}\footnotemark[1]}, 
  \textbf{Yuxuan Cheng\textsuperscript{4}}, 
  \textbf{Yuhan Pan\textsuperscript{1}} \\[0.1in] 
  \textbf{Miao Li\textsuperscript{1}}, 
  \textbf{Zhipeng Qiao\textsuperscript{1}}, 
  \textbf{Tao Feng\textsuperscript{5}}, 
  \textbf{Zhen Tao\textsuperscript{5}}, 
  \textbf{Dengji Zhao\textsuperscript{1}\thanks{Corresponding author.}} \\[0.15in]
  \textsuperscript{1}ShanghaiTech University \quad
  \textsuperscript{2}City University of Hong Kong \quad
  \textsuperscript{3}Peking University \\
  \textsuperscript{4}The Chinese University of Hong Kong, Shenzhen \quad
  \textsuperscript{5}Zhejiang University \\[0.15in]
  {wentaozh2001@gmail.com, hzhang2838-c@my.cityu.edu.hk, thinkerjiang@foxmail.com} \\
  {yuxuancheng1@link.cuhk.edu.cn, zhaodj@shanghaitech.edu.cn} \\[0.25in]
}

\begin{document}
\maketitle
\begin{abstract}
Large Language Models (LLMs) excel at multi-step reasoning, yet current parallel reasoning approaches often fail to distinguish the contributions of individual reasoning paths. Many paths may be redundant, misleading, or even detrimental, but outcome-level rewards assign uniform reward, leading to ambiguous learning signals and unstable training. 
We propose \textbf{Parallel Shapley}, a reinforcement learning framework that attributes fine-grained, path-level contributions in multi-path reasoning. Treating each path as a player in a cooperative game, we leverage Shapley values to quantify marginal contributions, using a generative reward model to evaluate path utilities and Monte Carlo sampling for efficient approximation. 
Experiments on mathematical reasoning benchmarks show that Parallel Shapley outperforms existing baselines while providing more stable and interpretable training. Our framework effectively ``fishes out the free riders,'' assigning reward proportionally and improving multi-path reasoning in LLMs.
\end{abstract}

\section{Introduction}
\textbf{Large Language Models (LLMs)} have achieved impressive performance across a wide range of natural language processing tasks~\cite{kaplan2020scaling,yang2024qwen2,vu2024gptvoicetasker}. 
More recently, inspired by human cognition~\cite{lewis2025brute}, advances in \emph{deep reasoning}~\cite{sun2024surveyreasoningfoundationmodels} have significantly improved the logical consistency and multi-step reasoning capabilities of LLMs, as exemplified by DeepSeek-R1, which shows that increasing reasoning depth during inference leads to substantial gains on challenging reasoning tasks~\cite{guo2025deepseek,shao2024deepseekmath}.

This success of such human-inspired deep reasoning models~\cite{guo2025deepseek,shao2024deepseekmath} naturally raises a further question: \textit{beyond increasing reasoning depth, can models better emulate how humans explore \textbf{multiple candidate solutions} and synthesize them into final decisions?}
Recent studies investigate this idea through \emph{test-time scaling (TTS)}~\cite{zhang2025surveytesttimescalinglarge}, where multiple candidate outputs are generated in parallel and evaluated via metric Pass@$k$ to reflect whether a correct solution exists among the trajectories.
However, such approaches mainly estimate an upper bound of model performance~\cite{Brown2024LargeLanguageMonkeys}, which treat reasoning trajectories independently.
In contrast, recent models~\cite{comanici2025gemini,guo2025deepseek,wen2025parathinker} move beyond by explicitly aggregating complementary information across parallel trajectories---referred to as \emph{\textbf{Parallel Reasoning}}---thereby mitigating premature convergence to suboptimal solutions. 
This paradigm directly responds to the above question and motivates the research in this work.

Early \emph{Parallel Reasoning} work~\cite{Yang2025MultiverseYL,Chen2025ASPDUA,wen2025parathinker}, employs supervised fine-tuning (SFT) methods to imitate parallel thinking by learning from pre-constructed trajectories but often falls short on deep reasoning.
More recently, motivated by the potential of on-policy reinforcement learning, Parallel-R1~\cite{Zheng2025ParallelR1TP} proposes a framework to teach models to parallel thinking and reasoning via Group Relative Policy Optimization~\cite{shao2024deepseekmath}.
However, existing SFT and RL-based approaches, especially Parallel-R1, \textbf{still lack fine-grained modeling of individual reasoning paths}. Since parallel paths may be complementary or even conflicting, explicitly modeling interactions and differentiating path contributions at the process level is crucial for improving inference-time reasoning quality.
\begin{figure}
    \centering
    \includegraphics[width=1.0\linewidth]{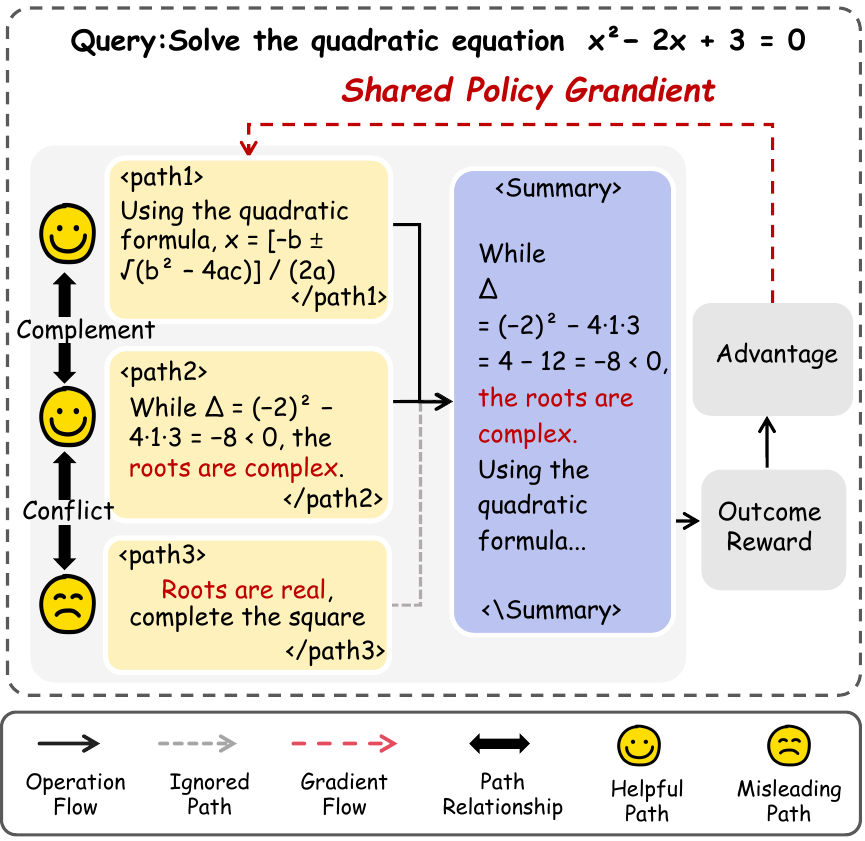}
    \vspace{0pt}
    \caption{Misleading path and helpful path shared policy gradient without recognizing. When solving the problem, $path_1$ and $path_2$ provide complementary and helpful solutions, while $path_3$ misleads the prediction.}
    \label{fig:intro}
    \vspace{0pt} 
\end{figure}

This limitation motivates our approach: we propose a \textbf{Parallel Thinking RL framework} that enables fine-grained, path-level control over reasoning trajectories. 
Nevertheless, effectively designing such a framework remains the following challenge: 

\noindent\textit{\textbf{Challenge: Quantifying path-level contributions and designing corresponding rewards.}}
In parallel reasoning, reasoning trajectories are inherently heterogeneous and \textbf{not independent}: some provide complementary and critical information that jointly improves the final prediction, while others introduce redundant, conflicting, or even detrimental evidence.
Existing parallel thinking RL frameworks~\cite{Zheng2025ParallelR1TP} primarily compute rewards at the aggregated outcome level and uniformly assign them to all reasoning trajectories.
However, since the summary is inherently \textbf{a compressed representation of multiple paths}, outcome-level rewards fail to reflect the true marginal contribution of each individual path.
As a result, erroneous, redundant, or misleading reasoning trajectories may still receive positive policy gradient updates when their information is absorbed or overridden during summarization---an instance of reward hacking~\cite{shihab2025detecting}.
More fundamentally, under \textbf{intra- and inter-path dependencies}, the marginal contribution of a single path cannot be meaningfully evaluated in isolation, as its usefulness is defined by its interaction with other paths, whether complementary or conflicting.

To address this challenge, we propose \textbf{Parallel Shapley}, a fine-grained reinforcement learning framework for multi-path reasoning via Group Relative Policy Optimization.
In implementation, we design a \textbf{hierarchical reward mechanism}: an \emph{outcome reward} at summary level based on correctness, and a \emph{process reward} at path level that captures each reasoning trajectory's marginal contribution.
To compute process rewards, we formalize the parallel reasoning process as a cooperative game, where each reasoning path is treated as a player.
We employ a \emph{Generative Reward Model} to evaluate the utility of different path combinations, and leverage \emph{Shapley values} as process rewards to quantify each path's contribution to the final outcome.
To circumvent the exponential complexity of exact Shapley value computation, we adopt Monte Carlo sampling to efficiently approximate marginal contributions by sampling permutations of paths, reducing the otherwise intractable combinatorial cost to a practical and scalable estimation procedure. Extensive experiments on multiple benchmarks demonstrate that \textbf{Parallel Shapley} consistently outperforms strong baselines, achieving an average of \textbf{42.3\%} improvement on Pass@16 compared to the SOTA parallel-thinking model, while exhibiting improved training stability and faster convergence.
In summary, our contributions are as follows:
\begin{itemize}[leftmargin=*]
    \item We propose \textbf{Parallel Shapley}, the first framework to introduce cooperative game theory and Shapley value attribution into parallel thinking, enabling fine-grained quantification of path-level marginal contributions.
    \item We design a Monte Carlo Shapley-based reward estimation method combined with a \emph{Generative Reward Model} to disentangle path-level process rewards from outcome rewards, addressing the intrinsic difficulty of reward attribution in summary-based parallel reasoning frameworks.
    \item Our approach demonstrates substantial and consistent performance gains across diverse benchmarks, accompanied by enhanced optimization stability and accelerated learning dynamics.
\end{itemize}

\section{Related Work}
\subsection{Parallel Reasoning}
Parallel Reasoning can be broadly categorized into inference-time and training-based approaches.
Inference-time methods explore multiple reasoning paths through brute-force sampling~\cite{Wang2022SelfConsistency,Brown2024LargeLanguageMonkeys} or heuristic-guided tree search, such as Tree of Thoughts~\cite{Yao2023Tree} and Monte Carlo Tree Search~\cite{Zhang2024Accessing}.
These approaches \textbf{rely on fixed search strategies or hand-crafted heuristics} and do not adapt the model's reasoning behavior through learning.
Training-based methods aim to internalize parallel reasoning capabilities within the model.
Under SFT, Multiverse~\cite{Yang2025MultiverseYL} models parallel generation via latent path merging decisions; ASPD~\cite{Chen2025ASPDUA} exploits intrinsic parallelism for adaptive decoding; and ParaThinker~\cite{wen2025parathinker} introduces a native parallel paradigm to scale test-time computation.
More recently, Parallel-R1~\cite{Zheng2025ParallelR1TP} adopts a reinforcement learning framework with a progressive curriculum to encourage multi-path exploration.
However, existing training-based approaches predominantly \textbf{rely on outcome-level rewards shared across all reasoning paths}, leaving fine-grained reward assignment among parallel trajectories underexplored.

\subsection{Reinforcement Learning for LLMs}
Reinforcement learning (RL) has emerged as a powerful paradigm for improving LLM alignment and reasoning, with recent work such as OpenAI-O1~\cite{openaio1} and DeepSeek-R1~\cite{guo2025deepseek} demonstrating strong performance on complex tasks. 
Methods like PPO~\cite{schulman2017proximal} have been widely adopted, but require multiple rounds of on-policy optimization. To improve efficiency, alternatives such as Direct Preference Optimization (DPO)~\cite{rafailov2024directpreferenceoptimizationlanguage} bypass reward modeling by directly optimizing preference loss, but suffer from off-policy bias. Group Relative Policy Optimization (GRPO)~\cite{shao2024deepseekmath} offers a compromise by avoiding learned value networks and leveraging intra-group comparisons for relative rewards.
Building on these, recent RL methods specifically designed for multi-step reasoning and tool interaction include AEPO~\cite{dong2025agentic}, which addresses entropy imbalance and rollout collapse in multi-round agents; ARPO~\cite{dong2025agentic2}, which guides exploration of uncertain high-entropy tool calls using advantage attribution; and MEM1~\cite{zhou2025mem1}, which synergizes memory and reasoning to support long-horizon multi-step interactions. More recently, DeepSearch~\cite{wu2026deepsearch} addresses the exploration bottleneck
in RLVR by integrating Monte Carlo Tree Search into the training loop. Despite these advances, existing RL approaches still face key challenges: \textbf{coarse-grained action modeling, trajectory-level scalar rewards, limited exploration in more broad dimensions}.

\section{Parallel Shapley Construction}
In this section, we present Parallel Shapley, which consists of three key components: \ding{182} tag-guided parallel reasoning structure using \texttt{<Path>} and \texttt{<Summary>} blocks (Section~\ref{sec:pre}), \ding{183} Shapley value computation via Monte Carlo sampling to quantify path-level marginal contributions (Section~\ref{sec:shapley_based}), and \ding{184} outcome and path-level rewards two-level reward strategy for reinforcement training in Section~\ref{sec:rl_train}. More details of Proofs, Prompt are in Appendix~\ref{sec:outcome_reward_bias},\ref{apd:prompt}.

\begin{figure*}
    \centering
    \includegraphics[width=1\linewidth]{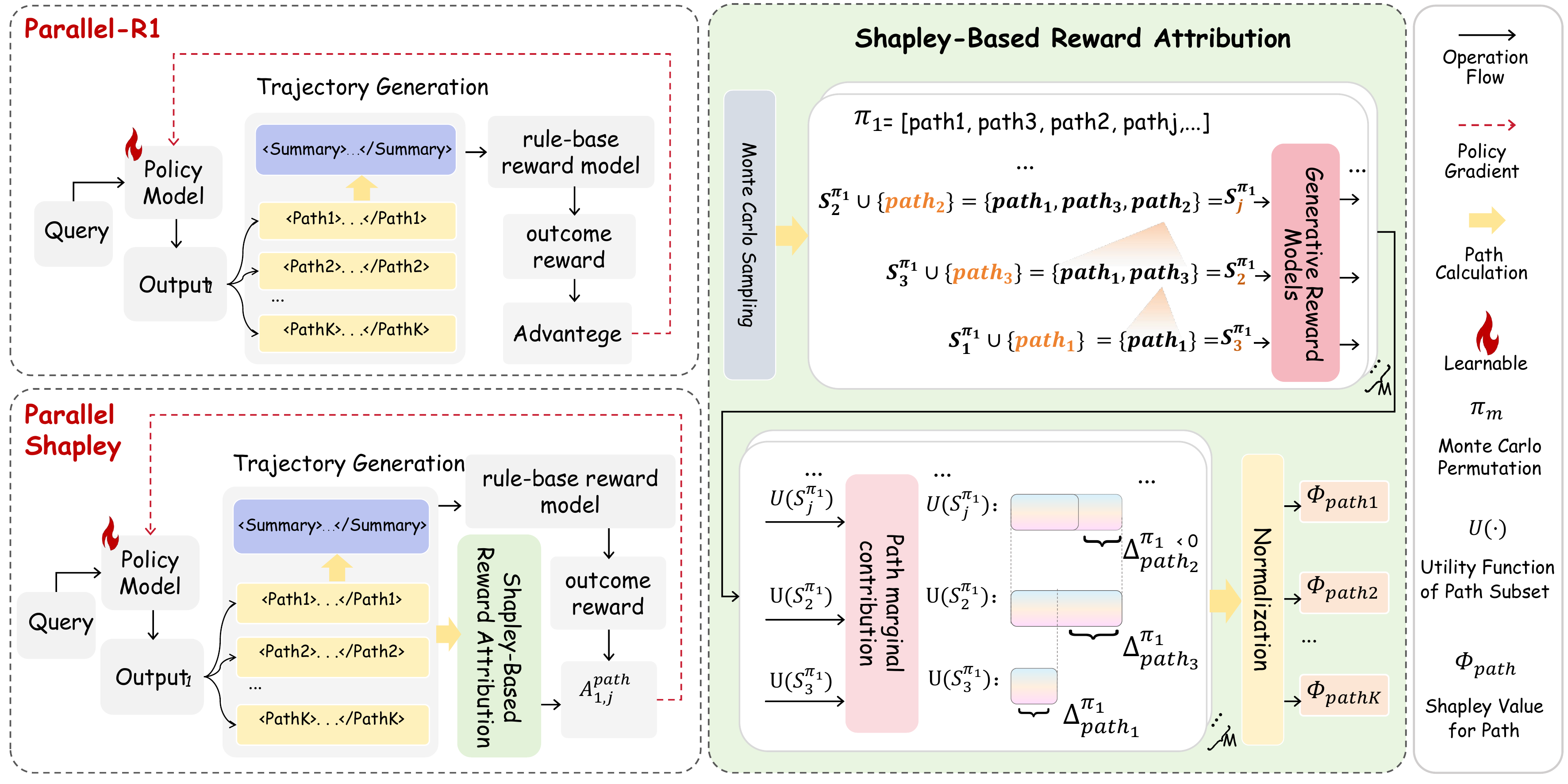}
    \vspace{0pt}
    \caption{\textbf{Overview of Parallel Shapley Pipeline.} \textbf{Left:} Comparison of training frameworks. Parallel-R1 assigns uniform outcome rewards to all paths, while Parallel Shapley computes individualized marginal contributions for each reasoning path. \textbf{Right:} The Shapley-based reward attribution mechanism. We sample $M$ path permutations via Monte Carlo, compute marginal contributions $\Delta_j^{\pi_m}$ under each ordering, and aggregate them to obtain $\phi_{\text{path}_j}^{\text{MC}}$ as the path-level reward for reinforcement learning.}
    \label{fig:method}
    \vspace{0pt} 
\end{figure*}

\subsection{Parallel Thinking Behavior Formulation}
\label{sec:pre}

Following the prior works~\cite{Yang2025MultiverseYL,Zheng2025ParallelR1TP}, we introduce these three control HTML-like tags:
\texttt{<Parallel>...</Parallel>},\texttt{<Path>...</Path>} and \texttt{<Summary>...</Summary>} to explicitly elicit parallel reasoning behaviors in large language models.
At the beginning of inference, the model performs standard autoregressive generation along a main reasoning chain.
\ding{182} Once the \texttt{<Parallel>} tag is generated, the model temporarily suspends the main chain and simultaneously expands multiple reasoning paths, each enclosed within an independent \ding{183} \texttt{<Path>...</Path>} block.
\ding{184}  After all parallel paths are completed, the model aggregates their contents into a concise \texttt{<Summary>...</Summary>} block,
where intermediate conclusions from different perspectives are synthesized to produce the final answer.

\subsection{Shapley-Based Reward Attribution}
\label{sec:shapley_based}
Under the parallel thinking formulation, the model generates $K$ reasoning trajectories within the blocks
\texttt{<Path>...</Path>}.
We denote the set of all reasoning paths for a given sample as
\(
\mathcal{P} \coloneqq \{\text{path}_1, \text{path}_2, \cdots,\) \(\text{path}_K\},
\) with $\text{path}_j$ denoting each path entry.

\paragraph{Marginal Contribution Calculation.}
Let $\mathcal{P}=\{\text{path}_1,\dots,$
$\text{path}_K\}$ denote the set of reasoning paths generated in the parallel thinking stage.
To quantify the contribution of each individual path to the final outcome, we define the score of a path using the \emph{Shapley value}, which measures its expected marginal contribution across all possible subsets of other paths. Formally, let $\mathbf{U} : 2^{\mathcal{P}} \rightarrow \mathbb{R}$ be a utility function defined over subsets of reasoning paths.
For any path $\text{path}_j \in \mathcal{P}$, its Shapley value  $\phi_{\text{path}_j}$ is:
\begin{equation*}\label{eq:shapley_definition}
\begin{aligned}
\phi_{\text{path}_j}
:=
\frac{1}{K}
\sum_{k=1}^{K}
\binom{K-1}{k-1}^{-1}
\sum_{\substack{S \subseteq \mathcal{P}\setminus{\text{path}_j} \\ |S| = k-1}}
\Delta_j^{S},
\end{aligned}
\end{equation*}
where  the utility function $\mathbf{U}(\cdot)$ evaluates the quality of a \emph{set of reasoning paths} when they jointly contribute to the final decision. By adding $\text{path}_j$ to $S$, the inner difference $\Delta_j^{S} = \mathbf{U}(S \cup \{\text{path}_j\})-\mathbf{U}(S)$ represents the marginal utility of $\text{path}_j$, and the averaging ensures fairness across all subset sizes.
However, due to complex dependencies and interactions among reasoning paths, $\mathbf{U}(\cdot)$ cannot be reliably estimated using verified reward signals such as RLVR~\cite{wang2025beyond}.
Therefore, we instantiate $\mathbf{U}(\cdot)$ using \emph{Generative Reward Model} (GRM)~\cite{mahan2024generativerewardmodels} to provide reward signals for generated outputs. Thus, all paths within a given subset can be holistically evaluated as:
\[
\mathbf{U}(S) \;=\; \mathrm{GRM}(S) \subset [0,5], \qquad S \subseteq \mathcal{P}.
\]
The GRM is prompted to produce scores in $[0,5]$ considering five dimensions and the detailed prompting template for GRM-based subset evaluation is provided in Appendix~\ref{apd:prompt}.

\paragraph{Monte Carlo Approximation.}
\label{sec:mc_shapley}

Exact computation of Shapley values requires evaluating the utility function $\mathbf{U}(\cdot)$ for all $2^K$ possible subsets of paths, which becomes computationally prohibitive as $K$ grows. To address this challenge, we adopt a Monte Carlo approximation based on the permutation formulation of Shapley values. The key insight is that instead of explicitly enumerating all subsets, we can estimate the expected marginal contribution by sampling random orderings of paths: for each ordering, a path's contribution is measured by comparing the utility before and after its insertion at its designated position.

Formally, let $\Pi$ denote the set of all permutations of the path set $\mathcal{P}$. We uniformly sample $M$ permutations $\{\pi_m\}_{m=1}^{M}$ from $\Pi$, where each permutation represents a hypothetical order in which paths are added to an initially empty coalition. For a given permutation $\pi_m = (\text{path}_{\pi_m(1)}, \text{path}_{\pi_m(2)}, \dots, \text{path}_{\pi_m(K)})$, let $S_j^{\pi_m}$ denote the set of paths appearing before $\text{path}_j$ in this ordering. The marginal contribution of $\text{path}_j$ under permutation $\pi_m$ is then computed as the difference in utility when $\text{path}_j$ joins the coalition $S_j^{\pi_m}$:
\[
\Delta_{j}^{\pi_m}
=
\mathbf{U}(S_j^{\pi_m} \cup \{\text{path}_j\})
-
\mathbf{U}(S_j^{\pi_m}).
\]
By averaging these marginal contributions across all sampled permutations, we obtain the Monte Carlo estimate of the Shapley value:
\[
\phi_{\text{path}_j}^{\mathrm{MC}}
=
\frac{1}{M}\sum_{m=1}^M \Delta_{j}^{\pi_m}.
\]
This approximation converges to the true Shapley value as $M \to \infty$, with the estimation error decreasing at rate $O(1/\sqrt{M})$. The resulting estimate $\phi_{\text{path}_j}^{\mathrm{MC}}$ serves as the path-level reward in our reinforcement learning framework Section~\ref{sec:rl_train}.

\subsection{Reinforcement Training}
\label{sec:rl_train}
After formulating the tag-guided parallel reasoning behavior through supervised fine-tuning following ~\cite{Yang2025MultiverseYL,Zheng2025ParallelR1TP}, we further optimize it via reinforcement learning. Our RL framework is on-policy and derived from PPO~\cite{schulman2017proximal}, following GRPO~\cite{shao2024deepseekmath}. 

Specifically, consider a rollout group consisting of $G_s$ trajectories
$\{y^{(p)}\}_{p=1}^{G_s}$ sampled from the old policy $\pi_{\theta_{\mathrm{old}}}$,
where each trajectory is a sequence of generated tokens.
Let $\mathcal{R}_G$ denote the set of all token-level rewards
$\{ r^{(p)}_i\}$
across the group.
For each token $x^{(p)}_i$ in trajectory $y^{(p)}$, we compute a normalized group-relative advantage as:
\begin{equation}
\hat{A}^{(p)}_i
=
\Bigl({
r^{(p)}_i - \texttt{mean}(\mathcal{R}_G) \Bigl) /
}{
\texttt{std}(\mathcal{R}_G)
}. \notag
\end{equation}
The GRPO optimization objective is defined as:
\begin{equation}
\footnotesize
\begin{aligned}
\mathcal{J}(\theta)
=
\mathbb{E}\Bigg[
\frac{1}{G_s} \sum_{p=1}^{G_s}
\frac{1}{|y^{(p)}|}
\sum_{i=1}^{|y^{(p)}|}
\texttt{Clip}\!\left(\tilde{z}_{i}^{(p)}, \hat{A}^{(p)}_i \right)
\Bigg]
-
\beta \mathbb{D}_{\mathrm{KL}}, 
\\
\texttt{Clip}\left(\tilde{z}_{i}^{(p)}, \hat{A}^{(p)}_i \right)
=
\min \Big(
\tilde{z}_{i}^{(p)} \hat{A}^{(p)}_i,
\operatorname{clip}(\tilde{z}_{i}^{(p)}, 1\pm\varepsilon)\hat{A}^{(p)}_i
\Big),
\notag
\end{aligned}
\end{equation} 
here importance ratio \(
\tilde{z}_{i}^{(p)}=\frac{
\pi_\theta(x^{(p)}_i \mid q, x^{(p)}_{<i})
}{
\pi_{\theta_{\mathrm{old}}}(x^{(p)}_i \mid q, x^{(p)}_{<i})
}\) denotes the probability ratio at the token level.
Term $\mathbb{D}_{\mathrm{KL}}(\pi_\theta \| \pi_{\mathrm{ref}})$
constrains the updated policy to remain close to a frozen reference policy
$\pi_{\mathrm{ref}}$.
Notably, all rewards, advantages, and policy updates in our framework
are defined at action level and applied at \emph{token level} with two aspects:
\begin{itemize}[leftmargin=*]
    \item \textbf{Outcome-level reward:} evaluates the quality of the final answer produced by output $i$. We extract the final answer from output $i$ and compare it with the ground truth $a_{\text{gold}}$ via exact string matching after normalization (removing whitespace, lowercasing, standardizing notation), with $\pm 1\%$ tolerance for numerical answers:
    \begin{equation*}
    R^{\text{out}}_i = \text{Acc}(\text{answer}_i, a_{\text{gold}}) \in \{0, 1\}.
    \end{equation*}
    \item \textbf{Path-level reward:} Monte Carlo Shapley value $\phi_{\text{path}_j}^{\mathrm{MC}}$ (Sec.~\ref{sec:mc_shapley}) quantifies the marginal contribution of each reasoning path, providing fine-grained reward assignment to individual paths.
\end{itemize}

\paragraph{Reward Combination.}
For each rollout, we parse the response into $K$ paths and a final \texttt{<Summary>} block. After normalizing both reward signals, we construct the token-level reward tensor as:
\begin{equation*}
r_t = \lambda_o R^{\text{out}} \cdot \mathbf{1}[t = T] + \lambda_p \sum_{j=1}^{K} \tilde{\phi}^{\mathrm{MC}}_j \cdot \mathbf{1}[t = e_j],
\end{equation*}
where $T$ is the last valid token of the whole response, $e_j$ is the last token of the corresponding \texttt{<Path\_j>} span, and $\tilde{\phi}^{\mathrm{MC}}_j$ denotes the normalized Shapley reward. In other words, the outcome reward is assigned to the last valid token of the full response, while each path-level reward is assigned to the closing token of its \texttt{<Path>} span. This follows the standard last-token reward assignment convention in RL for LLMs~\cite{sheng2024hybridflow}: since both final-answer correctness and path-level marginal utility can only be judged after the corresponding sequence or span is complete, the scalar reward is injected at its last valid token. After the reward tensor is constructed, GRPO performs group-normalized advantage estimation and applies token-level policy-gradient updates through the autoregressive likelihood objective.


\section{Experiments}

\begin{table*}[htbp]
\centering
\begin{adjustbox}{max width=\textwidth}
\begin{tabular}{llcccccccc}
\toprule
\multirow{2}{*}{\textbf{Type}} & \multirow{2}{*}{\textbf{Method}} & \multicolumn{2}{c}{\textbf{AIME25}} & \multicolumn{2}{c}{\textbf{AIME24}} & \multicolumn{2}{c}{\textbf{AMC23}} & \textbf{MATH} & \textbf{Avg.} \\
\cmidrule(lr){3-4} \cmidrule(lr){5-6} \cmidrule(lr){7-8} \cmidrule(lr){9-9} 
 & & Mean@16 & Pass@16 & Mean@16 & Pass@16 & Mean@16 & Pass@16 & Mean@1 & \\
\midrule
\multirow{7}{*}{\textbf{Baselines}} 
 & Qwen3-4B-Base & 1.3 & 10.2 & 2.9 & 16.5 & 8.1 & 51.2 & 13.9 & 6.6 \\
 \cmidrule(lr){2-10}
 & \textit{SFT + Parallel} & & & & & & & & \\
 & Parallel-SFT-Unseen & 5.2 & 20.9 & 8.5 & 26.7 & 41.7 & 80.1 & 71.5 & 31.7 \\
 \cmidrule(lr){2-10}
 & \textit{RL Approach} & & & & & & & & \\
 & GRPO (DAPO) & 14.8 & 32.4 & 18.5 & 30.6 & 63.6 & 85.1 & \textbf{83.5} & 45.1 \\
 & \quad + RL on GSM8K & 13.3 & 26.3 & 18.8 & 34.9 & 66.4 & 82.2 & 82.6 & 45.3 \\
 & Parallel-R1-Unseen & \textbf{17.7} & \underline{37.8} & 18.3 & 33.2 & \underline{69.7} & \underline{88.9} & 82.6 & \underline{47.1} \\
 \hline
\rowcolor{avgblue}
\textbf{Ours} & Parallel Shapley & \underline{17.2} & \textbf{50.0}& \textbf{20.1}& \textbf{63.3}&\textbf{72.8}& \textbf{92.5}& \underline{83.1} & \textbf{48.3} \\
 \midrule
\multirow{6}{*}{\textbf{Ablation}}
 & Leave One Out & 15.9 & 35.4 & 18.8 & 34.1 & 67.7 & 86.5 & 82.4 & 46.2 \\
 & Path Independent Evaluation & 16.8 & 36.1 & \underline{19.1} & 33.7 & 68.4 & 87.6 & 81.5 & 46.5 \\
 \cmidrule(lr){2-10}
 & \textit{GRM Sensitivity} & & & & & & & & \\
 & \quad w/ Qwen2.5-3B-Instruct & 18.1 & 43.2 & 19.4 & 46.8 & 71.2 & 90.1 & 82.9 & 47.4 \\
 & \quad w/ Qwen3-4B-Instruct   & 18.4 & 45.7 & 19.7 & 52.4 & 71.8 & 91.0 & 83.0 & 47.9 \\
 & \quad w/ Qwen2.5-7B-Instruct  & 18.6 & 47.3 & 19.9 & 58.2 & 72.3 & 91.8 & 83.0 & 48.1 \\
 
\bottomrule
\end{tabular}%
\end{adjustbox}
\caption{Performance comparison on mathematical reasoning benchmarks for the Qwen-3-4B-Base model trained under different parallel thinking configurations. We report Mean@16 and Pass@16 for AIME25, AIME24, and AMC23, while MATH is evaluated with Mean@1. GRM sensitivity rows share the same Parallel Shapley training setup ($K=4$) with different evaluator choices.}
\label{tab:main_exp}
\end{table*}

\begin{table}[htbp]
\centering
\small
\setlength{\tabcolsep}{3.5pt}

\resizebox{\columnwidth}{!}{
    \begin{tabular}{lccccc} 
    \toprule
    \textbf{Counts ($K$)} & \textbf{AIME 25} & \textbf{AIME 24} & \textbf{AMC 23} & \textbf{MATH} & \textbf{Avg.} \\ \midrule
    $K=2$ & 16.0 & 18.1 & 67.2& 80.4 & 45.4\\
    $K=3$ & 16.4 & 17.0 & 67.8& 81.2 & 45.6\\
    $K=4$ & \textbf{17.2} & \textbf{20.1} & \textbf{72.8} & \textbf{83.1} & \textbf{48.3} \\
    $K=5$ & 15.2 & 18.9 & 69.4& 81.6 & 46.2\\ \bottomrule
    \end{tabular}
}
\caption{\textbf{Ablation Study on Path Counts ($K$)}. Performance comparison across different number of paths.}
\label{tab:ablation_k}
\end{table}

\subsection{Experimental Setups}
\label{sec:setup}

\paragraph{Setup.}
We use Qwen3-4B-Base~\cite{Yang2025Qwen3TR} as our backbone model. Following Parallel-R1~\cite{Zheng2025ParallelR1TP}, we perform a cold-start initialization via supervised fine-tuning on Parallel-GSM8K~\cite{Zheng2025ParallelR1TP}, resulting in 230 gradient update steps. For reinforcement learning, we train on the DAPO dataset for only 40 gradient update steps using the GRPO algorithm, with a batch size of 256, a learning rate of 1e-6, and 8 rollouts per sample, without warm-up or learning rate scheduling. The maximum prompt length and response length are set to 2,000 and 3,000 tokens, respectively. We use \texttt{qwen-plus} as our default GRM. All experiments are built upon the VERL codebase~\cite{sheng2024hybridflow}, with its default training recipe adopted throughout without additional hyperparameter tuning.

\paragraph{Evaluation.}
At test time, we evaluate on four mathematical reasoning benchmarks: AIME24, AIME25, AMC23, and MATH~\cite{Hendrycks2021MeasuringMP}. We generate one response per problem on MATH. For the remaining three benchmarks, we draw 16 i.i.d.\ samples per problem at the same decoding temperature and report Mean@16 accuracy to reduce sampling variance, consistent with prior work~\cite{Zheng2025ParallelR1TP,wang2025beyond}. We additionally report Pass@16, computed as the fraction of problems for which at least one of the 16 sampled responses is correct, to reflect the upper bound of model capability.

\subsection{Main Results Analysis}
\label{sec:main_result}

We conduct experiments on AIME25, AIME24, AMC23, and MATH as shown in Table~\ref{tab:main_exp}. Overall, SFT with parallel-format prompting brings a large gain over the base model, but still lags behind RL, indicating that parallel thinking structure is helpful but not sufficient and on-policy optimization is necessary for strong mathematical reasoning.
\paragraph{Overall comparison.} Parallel Shapley achieves the best overall results while being substantially more training-efficient: \textbf{with only 40 RL steps}, our model already matches or exceeds Parallel-R1-Unseen trained \textbf{with 200 RL steps} on most settings. In particular, we obtain a clear improvement on the hardest benchmark AIME25, improving Pass@16 from $37.8 \to 50.0$ \textbf{(+32.3\%}). We also improve AIME24 Mean@16 from $18.3 \to 20.1$ (\textbf{+9.8\%}) and Pass@16 from $33.2 \to 63.3$ (\textbf{+90.7\%}), and raise AMC23 Mean@16 from $69.7 \to 72.8$ (\textbf{+4.4\%}) and Pass@16 from $88.9 \to 92.5$ (\textbf{+4.0\%}). On MATH, we achieve $83.1$ vs.\ $82.6$ for Parallel-R1-Unseen. These results support that Shapley-based path-level reward assignment provides denser and less biased learning signals than outcome-only shared rewards, effectively discouraging misleading trajectories and enabling faster convergence.

\paragraph{Pass@16 gains and the Mean@16 trade-off.}
The gains on Pass@16 are particularly striking, with an average improvement of \textbf{42.3\%} across three benchmarks, reaching $\textbf{50.0}$ on AIME25 and $\textbf{63.3}$ on AIME24. By rewarding each path proportionally to its marginal contribution, the model is steered toward generating \textbf{\emph{diverse and complementary}} reasoning trajectories, substantially increasing the probability of covering correct solutions across multiple attempts. Meanwhile, Mean@16 improvements are more moderate and, on AIME25, slightly below Parallel-R1 ($17.2$ vs.\ $17.7$). This is a natural consequence of the same mechanism: Shapley attribution encourages path specialization (Section~\ref{sec:path_comp}), so individual rollouts explore different angles rather than converging on a single answer, which elevates the upper bound (Pass@16) while tempering expected accuracy (Mean@16). We discuss bridging this gap via on-policy self-distillation in Section~\ref{sec:conclusion}.

\subsection{Ablation Study}

To investigate the effectiveness of path-level control, we conducted ablation studies comparing the following two strategies:
\begin{itemize}[leftmargin=*]
    \item \textbf{Leave-One-Out (LOO):} Measures each path's contribution by removing it from the full path set and evaluating the utility difference:
    \begin{equation*}
    \phi_{\text{path}_j}^{\text{LOO}} = \text{GRM}(\mathcal{P}) - \text{GRM}(\mathcal{P} - \{\text{path}_j\}).
    \end{equation*}
    This approach captures marginal contributions at the full coalition but does not incorporate weighting across different coalition sizes.
    
    \item \textbf{Independent Path Evaluation:} Evaluates each path in isolation without considering the presence of other paths, computing $\phi_{\text{path}_j}^{\text{Indep}} = \text{GRM}(\{\text{path}_j\})$.
\end{itemize}

The empirical results are presented in Table~\ref{tab:main_exp}. Both baselines yield suboptimal performance compared to Parallel Shapley. Independent path evaluation achieves the lowest performance (Avg. $46.5$), as it evaluates paths in isolation without modeling how they complement or conflict with each other in the final reasoning chain. LOO performs slightly worse (Avg. $46.2$) since measuring contributions only relative to the full path set overlooks how path values vary across different subset compositions---a key factor that Shapley values capture through weighted averaging over all possible coalitions. Parallel Shapley (Avg. $48.3$) outperforms both baselines by a substantial margin, demonstrating that explicit modeling of coalition-based marginal contributions is essential for accurately assigning reward in multi-path reasoning.


\paragraph{GRM Sensitivity Analysis.}
In our framework, the GRM serves exclusively as a \textbf{reward signal provider}:
it scores subsets of already-generated reasoning paths to estimate coalition-level
utility for Shapley attribution, and does not participate in path generation, provide
reasoning traces, or directly supervise individual path content. To verify that the observed gains are not an artifact of using a proprietary evaluator,
we replace \texttt{qwen-plus} with progressively weaker open-source alternatives under
the same $K=4$, 40-step RL setting. Results are reported in the lower block of
Table~\ref{tab:main_exp}. Two findings emerge. First, even the weakest evaluator
(Qwen2.5-3B-Instruct) yields consistent improvements over Parallel-R1-Unseen across
all benchmarks, confirming that the method does not rely exclusively on a strong
proprietary model. Second, stronger evaluators provide better subset-utility estimates
and further improve performance, indicating that GRM quality and Shapley attribution
are complementary: better utility estimates lead to more accurate marginal contribution
scores, which in turn produce cleaner training signals. Finally, using the identical \texttt{qwen-plus} evaluator (46.5 vs. 48.3 on average), the contrast between Independent Path Evaluation and Parallel Shapley isolates the effect of coalition‑aware attribution, demonstrating that improvements arise from Shapley‑based path interaction modeling instead of evaluator strength alone.

\paragraph{Training Dynamics.}
Figure~\ref{fig:training_curve} compares the training curves of Parallel Shapley, Parallel-R1, and GRPO over 40 RL steps. Parallel Shapley improves faster and reaches a higher final score ($0.652$) compared with Parallel-R1 ($0.609$) and GRPO ($0.591$). Notably, Parallel Shapley achieves stable convergence after approximately 30 steps, whereas Parallel-R1 and GRPO continue to improve slowly without plateauing within the same budget. This suggests that Shapley-based path-level reward attribution provides denser and more discriminative training signals, enabling the policy to identify and reinforce high-contribution paths more efficiently than uniform outcome-level rewards.
\begin{figure}
    \centering
    \includegraphics[width=\linewidth]{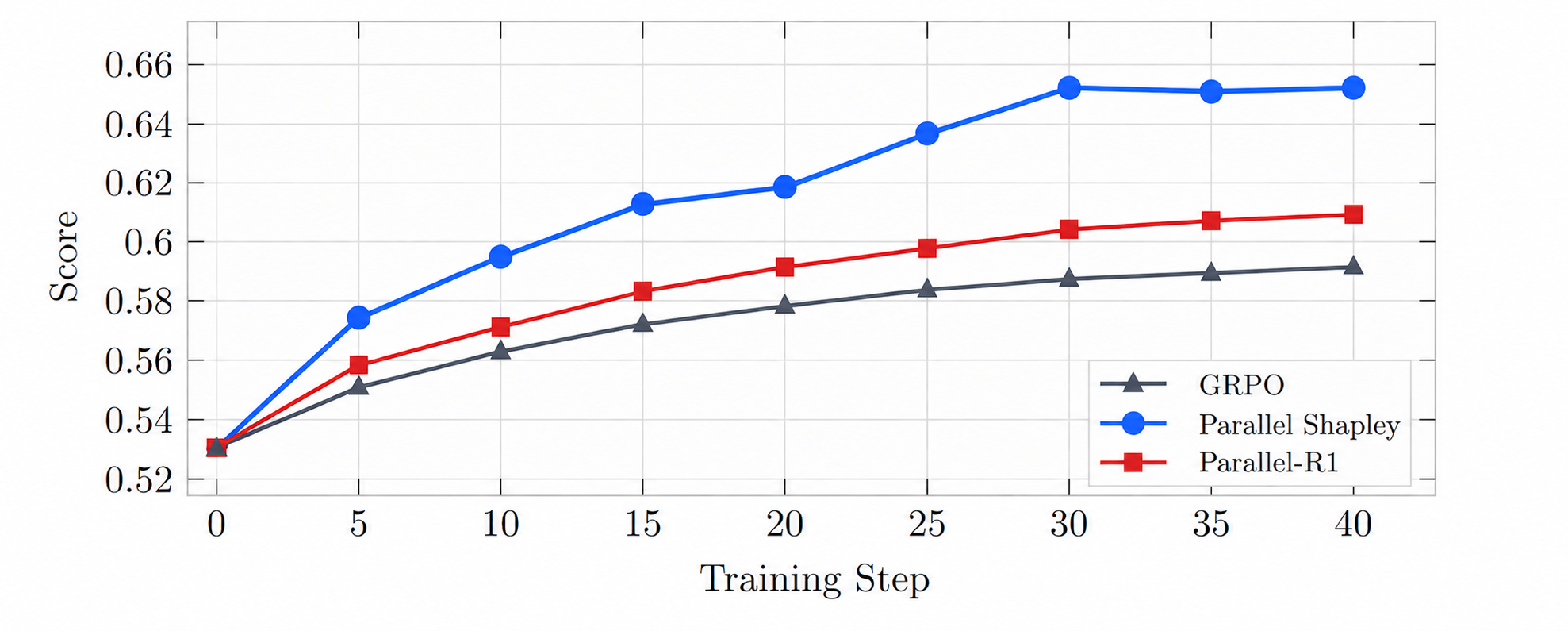}
    \caption{Training-curve comparison among Parallel Shapley, GRPO, and Parallel-R1 over 40 RL steps.}
    \label{fig:training_curve}
\end{figure}

\paragraph{Parameter Sensitivity Study.} We also conducted a parameter sensitivity study with different path counts  $K=2,3,4,5$. We found that the model performs best at $K=4$, achieving state-of-the-art results of $17.2$, $20.1$, and $72.8$ on AIME25, AIME24, and AMC23, respectively. Further increasing to $K=5$ degrades performance uniformly (Avg. $46.2$), indicating that excessive parallel paths introduce redundancy and computational overhead without proportional benefit. This suggests an optimal trade-off at $K=4$. The observed sensitivity underscores the importance of carefully tuning $K$ to balance exploration depth against diminishing returns, as overly large path sets can dilute the signal-to-noise ratio in Shapley value estimation.

\subsection{Path Complementarity Analysis}
\label{sec:path_comp}

To verify whether Parallel Shapley encourages complementary reasoning across paths, we conduct a Path Complementarity Analysis($K=4$): if paths are \emph{\textbf{truly complementary}}, each carries non-redundant information not
covered elsewhere, and masking any path should cause a sharp accuracy drop; if paths are \emph{\textbf{redundant}}, surviving paths suffice to reconstruct the answer and degradation should be mild. 
We therefore randomly mask a fraction of generated paths before the \texttt{<Summary>} stage and measure the resulting accuracy
loss as a diagnostic signal.
Results are shown in Figure~\ref{fig:dropout_path}.

\begin{figure}[t]
    \centering
    \includegraphics[width=\linewidth]{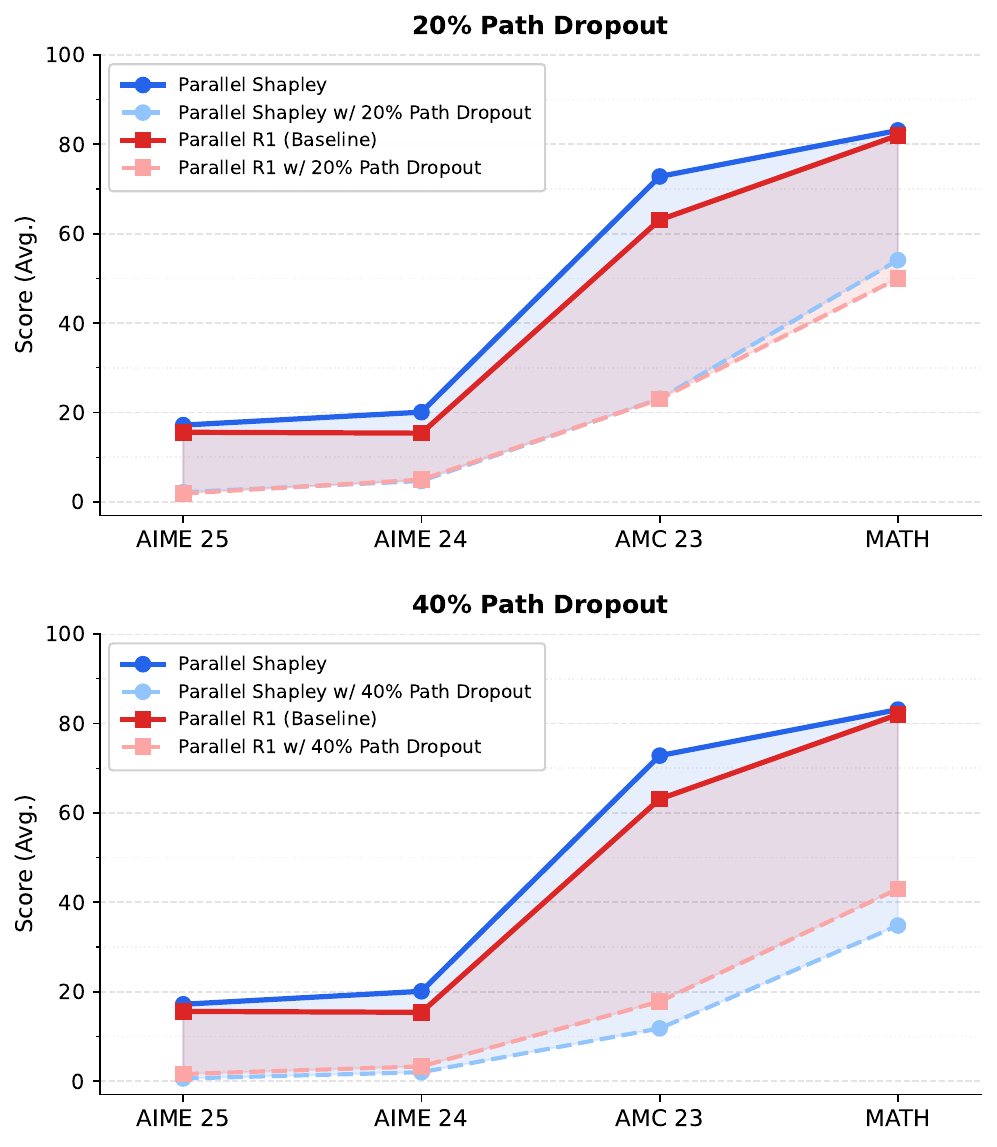}
    \caption{
        \textbf{Path dependency analysis ($K=4$).}
        Solid lines denote full-path baselines; dashed lines denote
        the corresponding masking conditions.
    }
    \label{fig:dropout_path}
\end{figure}

Across all benchmarks, both methods naturally degrade under 
path masking, yet Parallel Shapley degrades \emph{more 
sharply}---consistent with higher path complementarity.
Under 20\% masking, Parallel Shapley loses an average of \textbf{66.8\%} of its original score 
(e.g., AIME~25: $17.2 \to 2.2$, AIME~24: $20.1 \to 4.7$), 
whereas Parallel-R1 drops by \textbf{64.5\%} 
(AIME~25: $15.6 \to 1.9$, AIME~24: $15.4 \to 5.0$).
The gap widens considerably at 40\% masking: 
Parallel Shapley degrades by \textbf{82.1\%} on average, 
compared with \textbf{71.9\%} for Parallel-R1---a difference of 
10.2 percentage points.

This \textbf{differential degradation} directly reflects the two methods' structural
difference. Parallel Shapley rewards each path by its \textbf{marginal contribution},
steering the model toward a clear division of labor where each path
specializes in a distinct reasoning role---algebraic manipulation, case enumeration,
or boundary verification. Parallel-R1's uniform rewards, by contrast, produce
redundant paths that rehearse the same logic, so surviving paths still
suffice to reconstruct the answer. These results confirm that Parallel Shapley
successfully induces \textbf{complementary, non-overlapping reasoning} across paths.


\section{Conclusion and Future Work}
\label{sec:conclusion}
We propose \textbf{Parallel Shapley}, a reinforcement learning 
framework that formalizes each reasoning path as a player in a 
cooperative game and leverages \emph{Shapley values} with Monte 
Carlo approximation and a \emph{Generative Reward Model} to 
quantify path-level marginal contributions. Experiments on four 
mathematical reasoning benchmarks show consistent improvements 
over strong baselines with fewer training steps, yielding an 
average \textbf{42.3\%} Pass@16 improvement while maintaining 
competitive Mean@16. Preliminary multi-hop QA results 
(Appendix~\ref{apd:multihop}) further suggest generalizability 
beyond mathematical reasoning.

As discussed in Section~\ref{sec:main_result}, Shapley-based 
attribution encourages path specialization, elevating Pass@16 
while introducing a trade-off with Mean@16. To bridge this gap, 
we plan to explore on-policy 
self-distillation~\cite{wu2025native}, using the diverse 
multi-path trajectories as distillation targets to internalize 
synthesis behavior into single-generation performance. We also 
intend to extend Parallel Shapley to code generation, where 
parallel paths can naturally decompose into sub-modules and 
alternative implementations.

\section*{Limitations}
While Parallel Shapley demonstrates consistent improvements in path-level reward
assignment for parallel reasoning, several limitations remain. First, due to
computational constraints, our experiments are conducted on Qwen3-4B-Base, and
the method has not yet been evaluated on larger-scale models; we plan to extend
validation to larger backbones in future work. Second, although Monte Carlo
sampling reduces the combinatorial cost of exact Shapley computation, the number
of GRM queries still grows with the number of parallel paths $K$, which may
become a bottleneck at larger scales; we plan to explore more sample-efficient
approximation strategies such as stratified sampling to address this. Third,
the gap between Pass@16 and Mean@16 indicates that while Parallel Shapley
effectively improves the model's upper-bound reasoning capability, consistently
converting this into single-generation accuracy remains an open challenge; we
plan to investigate on-policy self-distillation as
a principled mechanism to bridge this gap in future work.
\bibliography{acl_latex}

\appendix

\newpage
\section{Outcome-Level Reward Misallocation in Parallel Reasoning}
\label{sec:outcome_reward_bias}

\begin{theorem}[Gradient Bias under Outcome-Only Rewards]
\label{thm:outcome_reward_bias}
Consider a parallel reasoning framework where $K$ reasoning paths 
$\mathcal{P} = \{\mathrm{path}_1, \ldots, \mathrm{path}_K\}$ 
are generated and aggregated into a summary $s$ that produces the final answer. 
Let $\color{blue}r_{\mathrm{out}}$ denote the outcome-level reward assigned based on answer correctness, 
and let $\color{orange}\pi_{\boldsymbol{\theta}}(\mathrm{path}_j \mid q_0)$ represent the policy generating path $j$ 
conditioned on query $q_0$. 
Under the standard policy gradient framework with uniform reward distribution, 
each path receives an equal share $\color{blue}\tilde{r}_j = \frac{r_{\mathrm{out}}}{K}$ regardless of its 
actual contribution. For a path $\mathrm{path}_A$ that provides \emph{misleading} or \emph{redundant} 
information (i.e., with negative true marginal contribution $\color{orange}\phi_A < 0$), 
the policy update still satisfies:
\begin{equation}\label{eq:outcome_bias_main}
\boldsymbol{\theta}_{t+1} 
= \boldsymbol{\theta}_t 
+ \mu \cdot \color{blue}\frac{r_{\mathrm{out}}}{K} \color{black}\cdot 
\nabla_{\boldsymbol{\theta}} \log \color{orange}\pi_{\boldsymbol{\theta}}(\mathrm{path}_A \mid q_0)\color{black},
\end{equation}
where $\mu > 0$ is the learning rate. This leads to a \textbf{probability increase} 
for path $\mathrm{path}_A$ proportional to the product of its own gradient magnitude 
and its alignment with other paths:
\begin{equation}\label{eq:prob_increase}
\begin{aligned}
& \frac{\color{orange}\pi_{\boldsymbol{\theta}_{t+1}}(\mathrm{path}_A) \color{black}- \color{orange}\pi_{\boldsymbol{\theta}_t}(\mathrm{path}_A)}{\color{orange}\pi_{\boldsymbol{\theta}_t}(\mathrm{path}_A)} \\
& \simeq 
\mu \cdot \color{blue}\frac{r_{\mathrm{out}}}{K} \color{black}\cdot
\big\|\nabla_{\boldsymbol{\theta}} \log \color{orange}\pi_{\boldsymbol{\theta}}(\mathrm{path}_A)\color{black}\big\|_2 \\
&\quad \cdot \sum_{j=1}^{K} 
\rho(\mathrm{path}_A, \mathrm{path}_j) \cdot
\big\|\nabla_{\boldsymbol{\theta}} \log \color{orange}\pi_{\boldsymbol{\theta}}(\mathrm{path}_j)\color{black}\big\|_2,
\end{aligned}
\end{equation}
For convenience, we denote the policy gradient for path $i$ as:$\color{orange}\mathbf{g}_i^{\boldsymbol{\theta}} \color{black}= \nabla_{\boldsymbol{\theta}} \log \color{orange}{\pi_{\boldsymbol{\theta}}(\mathrm{path}_i)}$, then the gradient similarity between two paths is measured by cosine similarity:

\begin{equation}\label{eq:gradient_similarity}
\rho(\mathrm{path}_i, \mathrm{path}_j) 
= \frac{\langle\color{orange}\mathbf{g}_i^{\boldsymbol{\theta}} \color{black}, \color{orange}\mathbf{g}_j^{\boldsymbol{\theta}} \color{black} \rangle}
{\|\color{orange}\mathbf{g}_i^{\boldsymbol{\theta}} \color{black}\|_2 \|\color{orange}\mathbf{g}_j^{\boldsymbol{\theta}} \color{black}\|_2}.
\end{equation}

\noindent Consequently, when $\color{blue}r_{\mathrm{out}} > 0$ (correct final answer), 
\emph{all paths receive positive gradient updates} even if some paths 
(like $\mathrm{path}_A$) contributed negatively or were overridden during summarization. 
This constitutes a form of \textbf{reward hacking}, where misleading paths 
``free-ride'' on the success of complementary paths.
\end{theorem}

\begin{proof}
We analyze the policy gradient dynamics under outcome-level reward allocation 
through a multi-stage derivation. Under the REINFORCE algorithm with outcome reward $\color{blue}r_{\mathrm{out}}$, 
the expected gradient for the entire path set $\mathcal{P}$ is:
\begin{equation}\label{eq:reinforce_gradient}
\nabla_{\boldsymbol{\theta}} \mathbb{E}[R_{\mathrm{total}}]
= \mathbb{E}\left[
\color{blue}r_{\mathrm{out}} \color{black}\cdot 
\sum_{j=1}^{K} \nabla_{\boldsymbol{\theta}} \log \color{orange}\pi_{\boldsymbol{\theta}}(\mathrm{path}_j \mid q_0)\color{black}
\right].
\end{equation}
Expanding this expression by linearity of expectation and gradient operators:
\begin{equation}\label{eq:gradient_expansion}
\begin{aligned}
&\nabla_{\boldsymbol{\theta}} \mathbb{E}[R_{\mathrm{total}}] \\
&= \sum_{j=1}^{K} \mathbb{E}\left[
\color{blue}r_{\mathrm{out}} \color{black}\cdot 
\nabla_{\boldsymbol{\theta}} \log \color{orange}\pi_{\boldsymbol{\theta}}(\mathrm{path}_j \mid q_0)\color{black}
\right] \\[3pt]
&= \sum_{j=1}^{K} \mathbb{E}\Bigg[
\color{blue}\frac{r_{\mathrm{out}}}{K} \color{black}\cdot K \\
&\quad\quad \cdot \nabla_{\boldsymbol{\theta}} \log \color{orange}\pi_{\boldsymbol{\theta}}(\mathrm{path}_j \mid q_0)\color{black}
\Bigg] \\[3pt]
&= K \sum_{j=1}^{K} \mathbb{E}\left[
\color{blue}\tilde{r}_j \color{black}\cdot
\nabla_{\boldsymbol{\theta}} \log \color{orange}\pi_{\boldsymbol{\theta}}(\mathrm{path}_j \mid q_0)\color{black}
\right],
\end{aligned}
\end{equation}
where we introduce the uniform per-path reward $\color{blue}\tilde{r}_j = \frac{r_{\mathrm{out}}}{K}$. 
This formulation reveals the \emph{implicit assumption of equal contribution}: 
each path is treated as if it contributes $\frac{1}{K}$ of the total outcome, 
regardless of its actual semantic utility. However, in practice, paths exhibit 
\emph{heterogeneous marginal contributions} $\color{orange}\{\phi_j\}_{j=1}^{K}$ 
where $\sum_{j=1}^{K} \color{orange}\phi_j = \color{blue}r_{\mathrm{out}}$ 
but individual $\color{orange}\phi_j$ can be negative (misleading paths) 
or near-zero (redundant paths).

Consider a single misleading path $\mathrm{path}_A$ with true contribution 
$\color{orange}\phi_A < 0$. The parameter update from Eq.~\eqref{eq:outcome_bias_main} induces:
\begin{equation}\label{eq:param_perturbation}
\Delta \boldsymbol{\theta} 
= \boldsymbol{\theta}_{t+1} - \boldsymbol{\theta}_t 
= \mu \cdot \color{blue}\frac{r_{\mathrm{out}}}{K} \color{black}\cdot 
\nabla_{\boldsymbol{\theta}} \log \color{orange}\pi_{\boldsymbol{\theta}_t}(\mathrm{path}_A)\color{black}.
\end{equation}
Recall that $\color{orange}\mathbf{g}_A^{\theta_t} \color{black}= \nabla_{\boldsymbol{\theta}} \log \color{orange}\pi_{\boldsymbol{\theta}_t}(\mathrm{path}_A)$ 
denotes the policy gradient.
Under the parameter update $\Delta \boldsymbol{\theta} = \mu \cdot \color{blue}\frac{r_{\mathrm{out}}}{K} \color{black} \cdot \mathbf{g}_A^{\theta_t}$,  
we apply Taylor expansion to the policy probability around $\boldsymbol{\theta}_t$:

\begin{equation}\label{eq:taylor_step1}
\begin{aligned}
 \color{orange}\pi_{\boldsymbol{\theta}_{t+1}}(\mathrm{path}_A) 
&= \color{orange}\pi_{\boldsymbol{\theta}_t}(\mathrm{path}_A) \color{black}
+ \\[3pt]
& \mu \cdot \color{blue}\frac{r_{\mathrm{out}}}{K}  \color{black}\cdot 
\nabla_{\boldsymbol{\theta}} \color{orange}\pi_{\boldsymbol{\theta}_t}(\mathrm{path}_A) \color{black}\cdot \color{orange} \mathbf{g}_A^{\theta_t} \color{black} + \mathcal{O}(\mu^2).
\end{aligned}
\end{equation}

Now we apply the \emph{log-derivative trick}. Recall that for any probability distribution:
\begin{equation}\label{eq:log_derivative_identity}
\nabla_{\boldsymbol{\theta}} \color{orange}\pi_{\boldsymbol{\theta}}(\mathrm{path}_A) 
= \color{orange}\pi_{\boldsymbol{\theta}}(\mathrm{path}_A) \color{black}\cdot 
 \color{orange} \mathbf{g}_A^{\theta_t} \color{black}.
\end{equation}
Substituting Eq.~\eqref{eq:log_derivative_identity} into Eq.~\eqref{eq:taylor_step1}:
\begin{equation}\label{eq:taylor_step2}
\begin{aligned}
\color{orange}\pi_{\boldsymbol{\theta}_{t+1}}(\mathrm{path}_A) \color{black}= 
 &\color{orange}\pi_{\boldsymbol{\theta}_t}(\mathrm{path}_A) \color{black}\cdot \\&\left[
1 + \mu \cdot \color{blue}\frac{r_{\mathrm{out}}}{K} \color{black}\cdot 
\big\| \color{orange} \mathbf{g}_A^{\theta_t} \color{black}\big\|_2^2
\right]
+ \mathcal{O}(\mu^2),
\end{aligned}
\end{equation}
where the inner product $\langle \bm{g}, \bm{g} \rangle = \|\bm{g}\|_2^2$ 
for any gradient vector $\bm{g}$. This reveals the critical flaw: 
since $\color{blue}r_{\mathrm{out}} > 0$ for correct answers and gradient norms are always positive, 
the bracketed term $[1 + \mu \cdot \frac{r_{\mathrm{out}}}{K} \cdot \| \color{orange} \mathbf{g}_A^{\theta_t} \color{black}\|_2^2] > 1$, 
guaranteeing $\color{orange}\pi_{\boldsymbol{\theta}_{t+1}}(\mathrm{path}_A) > \pi_{\boldsymbol{\theta}_t}(\mathrm{path}_A)$ 
\emph{regardless of the path's true contribution} $\color{orange}\phi_A < 0$.

The previous analysis considered only self-interaction 
($\nabla \log \pi(\mathrm{path}_A)$ with itself). 
To understand how misleading paths exploit correlations with high-quality paths, 
we must analyze the \emph{full gradient update} across all $K$ paths. 
From Eq.~\eqref{eq:gradient_expansion}, the aggregate gradient is:
\begin{equation}\label{eq:aggregate_gradient}
\bm{g}_{\mathrm{total}} 
= \color{blue}\frac{r_{\mathrm{out}}}{K} \color{black}\sum_{j=1}^{K} 
\nabla_{\boldsymbol{\theta}} \log \color{orange}\pi_{\boldsymbol{\theta}}(\mathrm{path}_j)\color{black}.
\end{equation}
The probability change for $\mathrm{path}_A$ under this aggregate update becomes:
\begin{equation}\label{eq:cross_path_step1}
\begin{aligned}
&\color{orange}\pi_{\boldsymbol{\theta}_{t+1}}(\mathrm{path}_A) \\
&\simeq \color{orange}\pi_{\boldsymbol{\theta}_t}(\mathrm{path}_A) \color{black}
+ \color{orange}\pi_{\boldsymbol{\theta}_t}(\mathrm{path}_A) \color{black}\cdot 
\left\langle 
 \color{orange} \mathbf{g}_A^{\theta_t} \color{black},
\mu \cdot \bm{g}_{\mathrm{total}}
\right\rangle \\[3pt]
&= \color{orange}\pi_{\boldsymbol{\theta}_t}(\mathrm{path}_A) \color{black}
+ \mu \cdot \color{blue}\frac{r_{\mathrm{out}}}{K} \color{black}\cdot \color{orange}\pi_{\boldsymbol{\theta}_t}(\mathrm{path}_A) \color{black}\cdot \\
& \left\langle 
 \color{orange} \mathbf{g}_A^{\theta_t} \color{black},
\sum_{j=1}^{K}  \color{orange} \mathbf{g}_j \color{black}
\right\rangle.
\end{aligned}
\end{equation}
Expanding the inner product by linearity:
\begin{equation}\label{eq:cross_path_step2}
\begin{aligned}
\left\langle 
 \color{orange} \mathbf{g}_A^{\theta} \color{black},
\sum_{j=1}^{K} \color{orange} \mathbf{g}_j^{\theta} \color{black} \right\rangle 
&= \sum_{j=1}^{K} 
\left\langle 
\color{orange} \mathbf{g}_A^{\theta} \color{black},
\color{orange} \mathbf{g}_j^{\theta} \color{black}
\right\rangle \\
&= \sum_{j=1}^{K} 
\big\|\color{orange} \mathbf{g}_A^{\theta} \color{black}\big\|_2 \cdot
\big\|\color{orange} \mathbf{g}_j^{\theta} \color{black}\big\|_2 \cdot
\rho,
\end{aligned}
\end{equation}
where we introduced the cosine similarity $\rho$
from Eq.~\eqref{eq:gradient_similarity}. Combining Eq.~\eqref{eq:cross_path_step1} 
and Eq.~\eqref{eq:cross_path_step2}:
\begin{equation}\label{eq:cross_path_final}
\begin{aligned}
&\frac{\color{orange}\pi_{\boldsymbol{\theta}_{t+1}}(\mathrm{path}_A) \color{black}- 
\color{orange}\pi_{\boldsymbol{\theta}_t}(\mathrm{path}_A)}
{\color{orange}\pi_{\boldsymbol{\theta}_t}(\mathrm{path}_A)} \\[3pt]
&\simeq \mu \cdot \color{blue}\frac{r_{\mathrm{out}}}{K} \color{black}\cdot 
\big\| \color{orange} \mathbf{g}_A^{\theta} \color{black}\big\|_2 
\sum_{j=1}^{K} 
\rho(\mathrm{path}_A, \mathrm{path}_j) \cdot 
\big\| \color{orange} \mathbf{g}_j^{\theta} \color{black}\big\|_2.
\end{aligned}
\end{equation}
The probability increase of $\mathrm{path}_A$ is a weighted sum of its 
gradient similarities $\rho(\mathrm{path}_A, \mathrm{path}_j)$ with \emph{all other paths}. 
Critically, if $\mathrm{path}_A$ shares surface-level linguistic features 
(e.g., similar phrasing, common reasoning templates) with high-quality paths, 
their gradient vectors will align ($\rho > 0$) even if $\mathrm{path}_A$'s 
semantic contribution is negative. This creates a \emph{spurious correlation} 
where misleading paths ``free-ride'' on the success of complementary paths 
through gradient space proximity rather than semantic utility.
To formalize the long-term impact of reward misallocation, 
we introduce the \textit{\textbf{reward allocation error}}:
\begin{equation}\label{eq:reward_error}
\varepsilon_j 
= \color{blue}\tilde{r}_j \color{black}- \color{orange}\phi_j 
\color{black}= \color{blue}\frac{r_{\mathrm{out}}}{K} \color{black}- \color{orange}\phi_j\color{black},
\end{equation}
which measures the discrepancy between the uniform reward $\color{blue}\tilde{r}_j$ 
and the true marginal contribution $\color{orange}\phi_j$. 
For a misleading path with $\color{orange}\phi_A < 0$, 
we have $\varepsilon_A = \frac{r_{\mathrm{out}}}{K} - \phi_A > \frac{r_{\mathrm{out}}}{K} > 0$ 
when $\color{blue}r_{\mathrm{out}} > 0$. Over $T$ training iterations, 
the cumulative gradient bias accumulates as:
\begin{equation}\label{eq:cumulative_bias_step1}
\begin{aligned}
\bm{B}_A^{(T)} 
&= \sum_{t=1}^{T} \varepsilon_A \cdot 
\color{orange} \mathbf{g}_A^{\theta_t} \color{black} \\[3pt]
&= \sum_{t=1}^{T} 
\left(
\color{blue}\frac{r_{\mathrm{out}}}{K} \color{black}- \color{orange}\phi_A
\right)
\color{orange} \mathbf{g}_A^{\theta_t} \color{black} \\[3pt]
&= \color{blue}\frac{r_{\mathrm{out}}}{K} \color{black}\sum_{t=1}^{T} 
\color{orange} \mathbf{g}_A^{\theta_t} \color{black}
\color{black}- \color{orange}\phi_A \color{black}\sum_{t=1}^{T} 
\color{orange} \mathbf{g}_A^{\theta_t} \color{black}\color{black}.
\end{aligned}
\end{equation}
Taking expectations and assuming stationarity of gradient statistics:
\begin{equation}\label{eq:cumulative_bias_step2}
\begin{aligned}
\mathbb{E}[\|\bm{B}_A^{(T)}\|_2] 
&\simeq \left|
\color{blue}\frac{r_{\mathrm{out}}}{K} \color{black}- \color{orange}\phi_A\color{black}
\right| \cdot T \cdot 
\mathbb{E}\left[\big\|\color{orange} \mathbf{g}_A^{\theta} \color{black}\big\|_2\right] \\[3pt]
&= \varepsilon_A \cdot T \cdot 
\mathbb{E}\left[\big\|\color{orange} \mathbf{g}_A^{\theta} \color{black}\big\|_2\right],
\end{aligned}
\end{equation}
where we used the triangle inequality and the fact that gradient directions 
exhibit bounded variance under typical neural network parameterizations. 
This reveals that the bias \textit{\textbf{grows linearly with training duration}} $T$, 
progressively amplifying the generation probability of misleading paths. 
The rate of amplification is proportional to both the magnitude of misallocation 
$\varepsilon_A$ and the gradient norm, which measures the path's 
``learnability'' in parameter space.

The cumulative bias in Eq.~\eqref{eq:cumulative_bias_step2} induces 
\emph{emergent optimization pathologies}. To formalize this, 
we analyze the implicit objective that the policy actually optimizes 
under outcome-only rewards. Ideally, we want to maximize the total utility:
\begin{equation}\label{eq:ideal_objective}
\max_{\boldsymbol{\theta}} \; \mathbb{E}_{\mathcal{P} \sim \pi_{\boldsymbol{\theta}}}
\left[\color{blue}r_{\mathrm{out}}(\mathcal{P})\right]
= \max_{\boldsymbol{\theta}} \; \mathbb{E}_{\mathcal{P} \sim \pi_{\boldsymbol{\theta}}}
\left[U(\mathcal{P})\color{black}\right],
\end{equation}
where $U(\mathcal{P})$ represents the true utility function that depends on 
the \emph{collective semantic contribution} of all paths in $\mathcal{P}$. 
In an ideal scenario, this utility should decompose as:
\begin{equation}\label{eq:utility_decomposition}
U(\mathcal{P}) = \sum_{j=1}^{K} \color{orange}\phi_j(\mathcal{P}_{-j})\color{black},
\end{equation}
where $\color{orange}\phi_j(\mathcal{P}_{-j})$ denotes the context-dependent marginal contribution of path $j$ given the presence of all other paths $\mathcal{P}_{-j} = \mathcal{P} \setminus \{\mathrm{path}_j\}$. 
Note that $\color{orange}\phi_j$ can be negative when path $j$ introduces 
misleading information or contradicts high-quality paths.

However, the uniform reward mechanism in Eq.~\eqref{eq:gradient_expansion} 
does not have access to individual marginal contributions. Instead, it implicitly optimizes:
\begin{equation}\label{eq:surrogate_objective}
\begin{aligned}
&\max_{\boldsymbol{\theta}} \; \sum_{j=1}^{K} 
\mathbb{E}_{\mathrm{path}_j \sim \pi_{\boldsymbol{\theta}}}
\left[\color{blue}\frac{r_{\mathrm{out}}}{K}\color{black}\right]\\
&= \max_{\boldsymbol{\theta}} \; 
\frac{1}{K} \sum_{j=1}^{K}
\mathbb{E}_{\mathrm{path}_j \sim \pi_{\boldsymbol{\theta}}}
\left[\color{blue}r_{\mathrm{out}}(\mathcal{P})\color{black}\right] \\[3pt]
&= \max_{\boldsymbol{\theta}} \; 
\mathbb{E}_{\mathcal{P} \sim \pi_{\boldsymbol{\theta}}}
\left[\color{blue}r_{\mathrm{out}}(\mathcal{P}\color{black})\right] \cdot 
\underbrace{\frac{1}{K} \sum_{j=1}^{K} 1}_{\text{uniform weight}}.
\end{aligned}
\end{equation}
The critical issue is that Eq.~\eqref{eq:surrogate_objective} treats 
$\color{blue}r_{\mathrm{out}}(\mathcal{P})$ as if it were \emph{equally attributable} 
to each path, effectively assuming:
\begin{equation}\label{eq:false_assumption}
\color{orange}\phi_j(\mathcal{P}_{-j}) \color{black}\approx 
\color{blue}\frac{r_{\mathrm{out}}(\mathcal{P})}{K} 
\color{black}\quad \forall j \in \{1, \ldots, K\}.
\end{equation}
This assumption is fundamentally violated in parallel reasoning, where paths exhibit heterogeneous contributions: some provide critical insights 
($\color{orange}\phi_j \gg \frac{r_{\mathrm{out}}}{K}$), 
some are redundant ($\color{orange}\phi_j \approx 0$), 
and some are misleading ($\color{orange}\phi_j < 0$).

The discrepancy between the ideal objective (Eq.~\eqref{eq:ideal_objective}) 
and the surrogate objective (Eq.~\eqref{eq:surrogate_objective}) 
creates a structurally misaligned optimization landscape. 
To quantify this misalignment, define the \emph{objective gap}:
\begin{equation}\label{eq:objective_gap}
\begin{aligned}
\Delta_{\mathrm{gap}} 
&= \left|
\mathbb{E}_{\mathcal{P} \sim \pi_{\boldsymbol{\theta}}}
\left[\sum_{j=1}^{K} \color{orange}\phi_j\color{black}\right]
- K \cdot \mathbb{E}_{\mathcal{P} \sim \pi_{\boldsymbol{\theta}}}
\left[\color{blue}\frac{r_{\mathrm{out}}}{K}\color{black}\right]
\right| \\[3pt]
&= \left|
\sum_{j=1}^{K} \mathbb{E}_{\mathcal{P} \sim \pi_{\boldsymbol{\theta}}}
\left[\color{orange}\phi_j \color{black}- \color{blue}\frac{r_{\mathrm{out}}}{K}\right]
\right| \\[3pt]
&= \left|
\sum_{j=1}^{K} \mathbb{E}[\varepsilon_j]
\right|,
\end{aligned}
\end{equation}
where $\varepsilon_j$ is the reward allocation error from Eq.~\eqref{eq:reward_error}. 
While the errors may partially cancel due to the constraint 
$\sum_{j=1}^{K} \color{orange}\phi_j = \color{blue}r_{\mathrm{out}}$, 
the \emph{gradient-level impact} does not cancel because each path's 
update direction is determined by its individual $\varepsilon_j$, 
not the aggregate sum.

This misalignment manifests as three distinct emergent pathologies:
\noindent\ding{182}
From Eq.~\eqref{eq:cross_path_final}, the policy learns that 
probability increase is proportional to 
$\sum_{j=1}^{K} \rho(\mathrm{path}_i, \mathrm{path}_j) \cdot \|\nabla_{\boldsymbol{\theta}} \log \pi(\mathrm{path}_j)\|_2$. 
Since all paths receive positive rewards when $\color{blue}r_{\mathrm{out}} > 0$, 
the policy is incentivized to \emph{maximize gradient alignment} 
with the average path distribution rather than maximizing semantic utility. 
Mathematically, the policy implicitly optimizes:
\begin{equation}\label{eq:alignment_objective}
\begin{aligned}
&\max_{\boldsymbol{\theta}} \; 
\mathbb{E}_{\mathrm{path}_i \sim \pi_{\boldsymbol{\theta}}}
\left[
\sum_{j=1}^{K} \rho(\mathrm{path}_i, \mathrm{path}_j)
\right]\\
&\quad \quad \quad \text{subject to} \quad
\mathbb{E}[\color{blue}r_{\mathrm{out}}\color{black}] > 0,
\end{aligned}
\end{equation}
which favors paths that ``look similar'' to existing paths in gradient space, 
regardless of whether they add new information.
\noindent\ding{183} Paths that paraphrase or rephrase existing reasoning receive 
$\color{blue}\tilde{r}_j = \frac{r_{\mathrm{out}}}{K} > 0$ 
despite contributing negligible marginal utility ($\color{orange}\phi_j \approx 0$). 
Since redundant paths naturally exhibit high gradient similarity 
$\rho(\mathrm{path}_{\mathrm{redundant}}, \mathrm{path}_{\mathrm{original}}) \to 1$ 
due to shared linguistic structure, they receive \emph{amplified positive updates} 
via the cross-path term in Eq.~\eqref{eq:cross_path_final}. 
Over training, this creates a positive feedback loop where 
the policy increasingly generates near-duplicate paths, 
reducing the effective diversity of the path set $\mathcal{P}$.
\noindent\ding{184} Most critically, paths with logical errors or misleading reasoning 
($\color{orange}\phi_j < 0$) still receive positive gradient updates 
as long as their surface features align with correct paths. 
From Eq.~\eqref{eq:cross_path_final}, even if $\mathrm{path}_A$ 
contradicts the final answer, it benefits from 
$\sum_{j \neq A} \rho(\mathrm{path}_A, \mathrm{path}_j) \cdot \|\nabla \log \pi(\mathrm{path}_j)\|_2 > 0$ 
when its gradient direction happens to align with high-quality paths. 
This is particularly problematic in neural networks, where gradient alignment 
often reflects \emph{surface-level feature similarity} (e.g., shared n-grams, 
common reasoning templates) rather than semantic coherence. 
The policy thus learns to generate paths that ``sound correct'' 
without verifying their logical validity.

\vspace{0.3em}
\noindent These three pathologies collectively constitute \textit{\textbf{reward hacking}}: 
the policy exploits structural properties of the reward mechanism rather than 
optimizing the true objective. 

We can now formally characterize 
the fundamental limitation of outcome-only reward mechanisms: Given a set of reasoning paths $\mathcal{P} = \{\mathrm{path}_1, \ldots, \mathrm{path}_K\}$ 
that collectively produce outcome reward $\color{blue}r_{\mathrm{out}}$, 
the \emph{reward assignment problem} asks: how should we distribute 
$\color{blue}r_{\mathrm{out}}$ among individual paths to align policy optimization 
with true semantic utility?
Outcome-only mechanisms answer this question with uniform allocation 
$\color{blue}\tilde{r}_j = \frac{r_{\mathrm{out}}}{K}$, which we have shown leads to:
\ding{182} \textit{ \textbf{Systematic bias} (Eq.~\eqref{eq:reward_error})}: 
Paths receive rewards misaligned with their true contributions 
$\color{orange}\phi_j$, with error $\varepsilon_j$ that can exceed 
$\frac{r_{\mathrm{out}}}{K}$ for misleading paths.
\ding{183} \textit{ \textbf{Cumulative amplification} (Eq.~\eqref{eq:cumulative_bias_step2})}: 
The bias accumulates linearly over training ($O(T)$), 
progressively distorting the policy's learned distribution.
\ding{184} \textit{ \textbf{Gradient-space exploitation} (Eq.~\eqref{eq:cross_path_final})}: 
The policy learns to maximize gradient alignment rather than semantic utility, 
leading to redundancy proliferation and contradiction tolerance.
\ding{185} \textit{ \textbf{Optimization landscape misalignment} (Eq.~\eqref{eq:objective_gap})}: 
The surrogate objective diverges from the ideal objective, 
causing high variance and slow convergence.

\vspace{0.3em}
\noindent The root cause is the information bottleneck 
in outcome-level rewards: the scalar $\color{blue}r_{\mathrm{out}}$ 
is a lossy compression of the $K$-dimensional contribution vector 
$\color{orange}\{\phi_1, \ldots, \phi_K\}$. 
During summarization, the model may filter out low-quality paths, 
correct their errors, or override their conclusions---yet all paths 
receive the same reward signal, making it impossible for the policy 
to learn which paths were actually useful.
\end{proof}

\section{Evaluation Prompt for Path-subset in Generative Reward Model}
\label{apd:prompt}
As shown in Table~\ref{tab:prompt-evaluator}, the prompt guides an evaluator through systematic assessment of mathematical or reasoning solutions against ground truth answers.

The evaluation follows two steps: \ding{182} Comprehensively analyzing the solution path including methods, formulas, calculations, and errors; \ding{183} Scoring against the reference answer. The rubric uses a 0-5 point scale across five dimensions: \textit{\textbf{Method Reliability}} (reasonable approach), \textit{\textbf{Formula Completeness}} (mathematical expressions included), \textit{\textbf{Computational Accuracy}} (correct calculations), \textit{\textbf{Result Correctness}} (matches ground truth within tolerance), and \textit{\textbf{Information Purity}} (no misleading content).

Special rules handle edge cases like off-topic responses, alternative valid methods, coincidentally correct answers with errors, and incomplete solutions. This structured approach enables nuanced evaluation beyond simple right-or-wrong judgments, which is essential for training reward models that distinguish high-quality reasoning from flawed paths.
\begin{table}[htp]
\centering
\renewcommand{\arraystretch}{1} 
\small
\begin{tabularx}{\linewidth}{@{} X @{}}
\toprule
\textbf{Prompt} \\
\midrule
You are a professional academic evaluator responsible for assessing the quality of the solution-path summary for \text{\{query\}}.\\[0.5em]
\midrule 

\hlg{\textbf{[Evaluation Procedure]}}\\[0.5em]
\textbf{Step 1:} Provide a comprehensive summary of \{path\}
\begin{itemize}[leftmargin=*, itemsep=-0.3em, topsep=0.3em]
  \item Summarize the methods and ideas proposed in the path
  \item Organize the formulas and computational steps shown in the path
  \item Extract intermediate results and the final answer from the path
  \item Identify any issues, mistakes, or erroneous information in the path
\end{itemize}

\textbf{Step 2:} Score with reference to \{ground\_truth\}
\begin{itemize}[leftmargin=*, itemsep=0.1em, topsep=0.3em]
  \item Use the final result in ground\_truth as the evaluation benchmark
  \item Assign a score based on how close the path's content/result is to the correct result
\end{itemize}\\[0.5em]

\hlg{\textbf{[Scoring Rubric]}}\\[0.2em]
\begin{itemize}[leftmargin=*, itemsep=0.1em, topsep=0.3em]
  \item Score range: [0, 5]
  \item Scoring dimensions (add 1 point for each satisfied item):
\end{itemize}\\[0.3em]

\textbf{a. Method Reliability (1 point)}\\[0.1em]
The path proposes a reasonable solution approach or computational method (it does not need to be identical to ground\_truth, but must be logically consistent and theoretically feasible).\\[0.3em]

\textbf{b. Formula Completeness (1 point)}\\[0.1em]
The path provides mathematical formulas/expressions used for computation (each key step has a corresponding equation).\\[0.3em]

\textbf{c. Computational Accuracy (1 point)}\\[0.1em]
The step-by-step calculations and intermediate derivations in the path are correct, with no errors.\\[0.3em]

\textbf{d. Result Correctness (1 point)}\\[0.1em]
Compared to ground\_truth: the final answer obtained in the path matches the standard answer (a $\pm$1\% numerical error is allowed).\\[0.3em]

\textbf{e. Information Purity (1 point)}\\[0.1em]
The path contains no incorrect distracting information, logical contradictions, or misleading content.\\[0.5em]

\hlp{\textbf{[Special Rules]}}\\[0.2em]
\begin{enumerate}[leftmargin=*, itemsep=0.1em, topsep=0.3em]
  \item If the path does not address \{query\} at all, assign a score of 0.
  \item If the path uses a different method than ground\_truth, dimension A can still receive the point as long as the logic is sound and it can yield the correct result.
  \item If the path contains clear mistakes but the final result is coincidentally correct (e.g., errors cancel out), award dimension D but do not award dimension E.
  \item If the computation in the path is unfinished, dimension D is automatically 0.
\end{enumerate}\\[0.5em]

\hlp{\textbf{[Output Format]}}\\[0.2em]
Output only the total score: \texttt{x points}.\\[0.3em]
\bottomrule
\end{tabularx}
\caption{Prompt for path-subset evaluation in Generative Reward Model.}
\label{tab:prompt-evaluator}
\end{table}

\section{Ethical Considerations}
Our research focuses on enhancing the mathematical reasoning capabilities of LLMs through reinforcement learning. All experiments are conducted exclusively on publicly available datasets and benchmarks, including Parallel-GSM8K, DAPO, AIME, AMC, and MATH, in accordance with their respective licenses and usage terms. We ensure full transparency in our training procedures and evaluation protocols. This work does not involve any personally identifiable information, human subjects, or animal subjects, and poses no ethical risks related to data privacy or misuse. The sole intent of this research is to advance progress in LLM reasoning.

This work focuses on enhancing mathematical reasoning in LLMs and poses no direct ethical risks. The datasets used are publicly available and contain no personally identifiable or sensitive information. As a general consideration, improvements in LLM reasoning ability could potentially be misused in educational contexts, such as automated problem-solving where independent reasoning is expected. We encourage responsible deployment of systems built upon this work.

\section{Further Experiment Details}

\subsection{Experimental Configuration}
The RL hyperparameter settings are described in Section~\ref{sec:setup}. For the cold-start SFT stage, we train on Parallel-GSM8K with a batch size of 128, a learning rate of 1e-5, a weight decay of 0.01, and a warm-up step ratio of 0.1 under a cosine learning-rate schedule, resulting in 58/230 gradient update steps. For Monte Carlo sampling in Shapley value computation, we perform full enumeration for path counts $K$ = 2, and set the number of samples to 5 for $K=3,4,5$, balancing estimation fidelity and computational overhead. The reward combination weights are set to $\lambda_o = 0.5$ and $\lambda_p = 0.5$ throughout all experiments.

\subsection{Datasets Statistics}
We introduce the datasets used in our experiments, all of which are publicly available. 
\begin{itemize}[leftmargin=*]
    \item \textbf{Parallel-GSM8K}~\cite{Zheng2025ParallelR1TP} is a cold-start dataset constructed from the GSM8K training set via parallel-thinking prompting, consisting of approximately 7,473 grade school math problems annotated with parallel thinking trajectories. It is used for supervised fine-tuning (SFT) in the cold-start stage of parallel thinking training, and will be released upon publication.
    \item \textbf{DAPO}~\cite{yu2025dapo} is a large-scale mathematical reasoning dataset containing 17,000 curated prompts, each paired with a verifiable integer answer. The problems are sourced from competition-level mathematics benchmarks and are designed to support reinforcement learning training for LLM reasoning.
    \item \textbf{AIME}~\cite{aime} (American Invitational Mathematics Examination) is an olympiad-level mathematical reasoning benchmark. Each annual edition consists of 30 problems with integer answers in the range $[0, 999]$. This work evaluates on AIME 2024 and AIME 2025.
    \item \textbf{AMC}~\cite{amc} (American Mathematics Competition) is a competition-level mathematical reasoning benchmark. The AMC 2023 edition (AMC23) collects problems from the 2023 AMC 10 and AMC 12 examinations, covering a broad range of pre-calculus topics.
    \item \textbf{MATH}~\cite{hendrycks2021measuringmathematicalproblemsolving} is a collection of 12,500 challenging competition mathematics problems spanning seven subject areas, with difficulty levels ranging from 1 to 5. It is widely used for evaluating mathematical problem-solving ability across diverse topics and difficulty levels.
\end{itemize}

\section{Computational Resources and Software Environment}
All experiments were conducted on a single server equipped with four NVIDIA H100 GPUs (80\,GB memory per GPU) and 503\,GB of system RAM, running Ubuntu 22.04.5 LTS. The software environment is based on Python 3.10.0 and Conda 23.5.2. We implemented our experiments using PyTorch 2.6.0 and HuggingFace Transformers 4.51.3; unless otherwise specified, all default settings were used. For the generative reward model (GRM), we use \texttt{qwen-plus} to evaluate our path-subset values.

\section{Runtime Analysis}
\label{apd:runtime}
In the cold-start stage, we performed 5 hours of supervised fine-tuning with 230 gradient steps. In the post-training stage, we ran reinforcement learning for 40 gradient steps. For different path counts $K=2,3,4,5$, the training times are $24$ hours $\times 1$, $24$ hours $\times 2$, $24$ hours $\times 3$, and $24$ hours $\times 3$, respectively. The reported time includes both the necessary GPU computation and the GRM-based evaluation of each path subset.

\paragraph{Cost-Performance Analysis.}
Table~\ref{tab:cost_performance} compares per-step training time and average Pass@16 between Parallel-R1 and Parallel Shapley under the same $K=4$, 40-step RL setting.

\begin{table}[h]
\centering
\small
\setlength{\tabcolsep}{10pt}
\renewcommand{\arraystretch}{1.2}
\begin{tabular}{lcc}
\toprule
\textbf{Method} & \textbf{Time / Step} & \textbf{Avg.\ Pass@16} \\
\midrule
Parallel-R1      & 0.94 min & 53.3 \\
Parallel Shapley & 1.52 min & 68.6 \\
\bottomrule
\end{tabular}
\caption{Cost-performance comparison under the same $K=4$, 40-step RL setting. Parallel Shapley increases per-step wall-clock time by $\sim$62\% while improving Avg.\ Pass@16 by 15.3 points.}
\label{tab:cost_performance}
\end{table}

Parallel Shapley increases per-step wall-clock time from 0.94 min to 1.52 min, while improving Avg.\ Pass@16 from 53.3 to 68.6. Although Shapley estimation requires evaluating multiple path subsets per step, these subset-evaluation requests are issued \emph{concurrently} in our implementation, so the wall-clock overhead is closer to the latency of a batched evaluator call rather than the sum of all subset evaluations. Importantly, this overhead is \textbf{training-only}: at inference time, Parallel Shapley uses the trained policy directly and requires no GRM calls.

\section{The Use of Large Language Models}
In this study, Large Language Models (LLMs) were employed as auxiliary tools to assist with language refinement and programming-related support. Their use was limited to enhancing grammatical accuracy, readability, and overall writing style, as well as offering general coding suggestions or debugging guidance. All LLM-assisted outputs were thoroughly examined and validated by the authors prior to adoption. The conception of the research, methodological design, and experiment results analyses were entirely carried out by the authors. LLMs did not contribute to the formulation of research ideas or the derivation of conclusions.

\section{Broader Applicability: Multi-Hop Question Answering}
\label{apd:multihop}

To examine whether Parallel Shapley generalizes beyond mathematical reasoning, we conduct additional experiments on three multi-hop QA benchmarks: 2WikiMultiHopQA~\cite{ho2020constructing}, HotpotQA~\cite{yang2018hotpotqa}, and MuSiQue~\cite{trivedi2022musique}. These datasets require the model to gather, connect, and verify evidence across multiple reasoning hops, making them a natural testbed for path-level diversity and complementarity.

Since Parallel-R1 was not evaluated on multi-hop QA in its original work, we compare against Qwen3-4B-Base, GRPO, and Parallel-R1 using the same cold-start SFT initialization and RL training protocol as in the main experiments ($K=4$, 40 RL steps). Results are shown in Table~\ref{tab:multihop}.

\begin{table}[htp]
\centering
\small
\setlength{\tabcolsep}{5pt}
\renewcommand{\arraystretch}{1.2}
\begin{tabular}{lcccc}
\toprule
\textbf{Method} & \textbf{2Wiki} & \textbf{HotpotQA} & \textbf{MuSiQue} & \textbf{Avg.} \\
\midrule
Qwen3-4B-Base    & 28.9 & 30.3 & 12.2 & 23.8 \\
GRPO             & 31.2 & 32.5 & 13.7 & 25.8 \\
Parallel-R1      & 33.5 & 36.4 & 15.1 & 28.3 \\
Parallel Shapley & \textbf{35.6} & \textbf{38.2} & \textbf{16.5} & \textbf{30.1} \\
\bottomrule
\end{tabular}
\caption{Performance on multi-hop QA benchmarks. All methods use the same Qwen3-4B-Base backbone and training protocol as the main experiments.}
\label{tab:multihop}
\end{table}

Parallel Shapley consistently outperforms all baselines across the three datasets, improving the average score from 28.3 (Parallel-R1) to 30.1 (+6.4\%). These results suggest that Shapley-based path-level reward attribution is not limited to mathematical reasoning: multi-hop QA tasks similarly benefit from complementary path generation, where different paths gather and verify distinct pieces of evidence rather than repeating the same retrieval logic.

\section{More Case Studies}
\label{apd:more_case}

We present three case studies highlighting the complementary reasoning behavior induced by Shapley-based training.

Table~\ref{tab:case-juice-equalization} shows a juice-equalization problem from the main text, where two parallel paths naturally decompose the problem into complementary sub-tasks: one establishing the quantitative target, the other performing the symbolic reasoning. This division of labor mirrors how humans approach multi-step problems and is a direct consequence of Shapley-based credit assignment, which discourages redundant paths and incentivizes genuinely distinct reasoning perspectives.

In Table~\ref{tab:case-divisibility}, the single \texttt{<Parallel>} 
block contains two complementary paths: Path~1 provides the theoretical 
justification via the Remainder Theorem, while Path~2 instantiates it 
with the concrete substitution $n=-2$ to obtain the divisor $39$. 
Neither path alone constitutes a complete argument---the principle 
requires computation to be actionable, and the computation requires 
the principle to be valid.

Table~\ref{tab:case-parabola} shows a more complex example involving 
parabola rotation, where the model autonomously decomposes the problem 
into sub-questions, each handled by a dedicated \texttt{<Parallel>} 
block. In the second block, Path~1 establishes the trigonometric 
transformation framework but does not identify how to proceed, while 
Path~2 supplies the missing step by recognizing the key condition that 
resolves the problem. Together, the two paths form a complete reasoning 
chain that neither could provide independently.

\begin{table*}[ht]
\centering
\small
\setlength{\tabcolsep}{8pt}
\renewcommand{\arraystretch}{1.2}

\begin{tabularx}{\linewidth}{@{}
  >{\RaggedRight\arraybackslash}X
@{}}
\toprule
\textbf{Problem:} Mrs.\ Jones fills the first three glasses completely but runs out of juice when the fourth
glass is only $\frac{1}{3}$ full. What fraction of a glass must Mrs.\ Jones pour
from each of the first three glasses into the fourth glass so that all four glasses
will have the same amount of juice? $m$ and $n$ are relatively prime positive integers.
What is $m+n$?
\\
\midrule
\textbf{Ground Truth:} \textbf{7}
\\
\midrule
\textbf{Output:}
Let the capacity of each glass be $1$ unit. The first three glasses each contain
$1$ unit and the fourth contains $\frac{1}{3}$ unit. The total juice is
$1 + 1 + 1 + \frac{1}{3} = \frac{8}{3}$ units, so the equal share per glass is
$\frac{8}{3} \div 4 = \frac{2}{3}$ units.
\\
\hlg{\texttt{<Parallel>}} \\[-0.9em]
\begin{itemize}[leftmargin=*,topsep=0.1em, noitemsep]
  \item \texttt{<Path>} Let $x$ be the fraction poured from each of the first three glasses. After pouring, glass $1$ has $1-x$, glass $2$ has $1-x$, glass $3$ has $1-x$, and glass $4$ has $\frac{1}{3}+3x$. Setting them equal: $1-x=\frac{1}{3}+3x$, so $\frac{2}{3}=4x$, giving $x=\frac{1}{6}$. \texttt{</Path>}
  \item \texttt{<Path>} The total juice is $\frac{8}{3}$ units. Dividing equally among four glasses gives $\frac{2}{3}$ units per glass. \texttt{</Path>}
\end{itemize}
\\[-1.3em]
\hlg{\texttt{</Parallel>}} \\[-0.1em]
\hlp{\texttt{<Summary>}} The parallel block demonstrates complementary approaches: algebraic equation solving and total-and-divide. The consistent answer gives the fraction to pour as $\frac{1}{6}$. \hlp{\texttt{</Summary>}} \\
From the equation $1 - x = \frac{1}{3} + 3x$, solving yields $x = \frac{1}{6}$. \\
\hlg{\texttt{<Parallel>}} \\[-0.9em]
\begin{itemize}[leftmargin=*, noitemsep, topsep=0.1em]
  \item \texttt{<Path>} From $1 - x = \frac{1}{3} + 3x$, we get $\frac{2}{3} = 4x$, so $x = \frac{2}{3} \div 4 = \frac{1}{6}$. \texttt{</Path>}
  \item \texttt{<Path>} After pouring $\frac{1}{6}$ from each of the first three glasses, glass $1$: $1 - \frac{1}{6} = \frac{5}{6}$, glass $2$: $\frac{5}{6}$, glass $3$: $\frac{5}{6}$, glass $4$: $\frac{1}{3} + 3 \times \frac{1}{6} = \frac{5}{6}$. All four glasses contain $\frac{5}{6}$ units. \texttt{</Path>}
\end{itemize}
\\[-1.3em]
\hlg{\texttt{</Parallel>}} \\[-0.1em]
\hlp{\texttt{<Summary>}} Both approaches confirm that the fraction to pour from each glass is $\frac{1}{6}$. \hlp{\texttt{</Summary>}} \\
Thus $\frac{m}{n} = \frac{1}{6}$, so $m = 1$, $n = 6$, and $m + n = 7$. \\
Final Answer: \textbf{7}
\\
\bottomrule
\end{tabularx}
\caption{Juice-equalization problem demonstrating complementary path decomposition: one path sets up the algebraic equation, the other establishes the quantitative target.}
\label{tab:case-juice-equalization}
\end{table*}

\begin{table*}[htbp]
\centering
\small
\setlength{\tabcolsep}{8pt}
\renewcommand{\arraystretch}{1.2}

\begin{tabularx}{\linewidth}{@{}
  >{\RaggedRight\arraybackslash}X
@{}}
\toprule
\textbf{Problem:} Find the sum of all positive integers $n$ such that
$n + 2$ divides the product $3(n + 3)(n^2 + 9)$.
\\
\midrule
\textbf{Ground Truth:} \textbf{49}
\\
\midrule
\textbf{Output:}
We seek all positive integers $n$ such that
$n+2 \mid 3(n+3)(n^2+9)$, i.e., the division leaves no remainder.
\\
\hlg{\texttt{<Parallel>}} \\[-0.9em]
\begin{itemize}[leftmargin=*, topsep=0.1em, noitemsep]
  \item \texttt{<Path>} Use the Remainder Theorem: dividing by $n+2$ is equivalent to evaluating the polynomial at $n=-2$. \texttt{</Path>}
  \item \texttt{<Path>} Let $f(n)=3(n+3)(n^2+9)$. Then $f(-2)=3(1)(13)=39$, so $n+2$ must divide $39$. \texttt{</Path>}
\end{itemize}
\\[-1.3em]
\hlg{\texttt{</Parallel>}} \\[-0.1em]
\hlp{\texttt{<Summary>}} The divisibility condition reduces to requiring that $n+2$ is a positive divisor of $39$. \hlp{\texttt{</Summary>}} \\
The positive divisors of $39$ are $1,3,13,39$, yielding
$n=-1,1,11,37$. Discarding the negative value leaves $n=1,11,37$.
Each remaining value satisfies the original divisibility condition upon direct verification.
Summing the valid values gives $1+11+37=49$.
\\
Final Answer: \textbf{49}
\\
\bottomrule
\end{tabularx}
\caption{Divisibility problem where parallel paths contribute 
complementary roles: one supplying the theoretical principle 
(Remainder Theorem) and the other the concrete computation.}
\label{tab:case-divisibility}
\end{table*}

\begin{table*}[htbp]
\centering
\small
\setlength{\tabcolsep}{8pt}
\renewcommand{\arraystretch}{1.2}

\begin{tabularx}{\linewidth}{@{}
  >{\RaggedRight\arraybackslash}X
@{}}
\toprule
\textbf{Problem:} The parabola with equation $y = x^2 - 4$ is rotated $60^\circ$ counterclockwise
around the origin. The unique point in the fourth quadrant where the original parabola and its image
intersect has $y$-coordinate $\frac{a - \sqrt{b}}{c}$, where $a$, $b$, and $c$ are positive integers,
and $a$ and $c$ are relatively prime. Find $a + b + c$.
\\
\midrule
\textbf{Ground Truth:} \textbf{62}
\\
\midrule
\textbf{Output:}
The parabola $y = x^2 - 4$ is rotated $60^\circ$ counterclockwise around the origin.
To find the intersection in the fourth quadrant, I need to determine the equations of the
rotated parabola and the original one, then solve for their intersection.
\\
\hlg{\texttt{<Parallel>}} \\[-0.9em]
\begin{itemize}[leftmargin=*, topsep=0.1em, noitemsep]
  \item \texttt{<Path>} Geometric interpretation: the rotation transforms $(x,y)$ to
  $\bigl(x\cos60^\circ - y\sin60^\circ,\; x\sin60^\circ + y\cos60^\circ\bigr)$.
  Substituting $y = x^2-4$ and simplifying with $\sin60^\circ=\frac{\sqrt{3}}{2}$,
  $\cos60^\circ=\frac{1}{2}$ yields the rotated equation
  $Y = \frac{X^2}{2} + \frac{\sqrt{3}}{2}X - 2$. \texttt{</Path>}
  \item \texttt{<Path>} Rotation matrix approach: $(x',y')=\bigl(\frac{1}{2}x-\frac{\sqrt{3}}{2}y,\;
  \frac{\sqrt{3}}{2}x+\frac{1}{2}y\bigr)$. Substituting $y=x^2-4$ into these expressions
  gives the rotated parabola, whose intersection with the original can then be solved
  algebraically. \texttt{</Path>}
  \item \texttt{<Path>} Inverse substitution: express $x,y$ in terms of $x',y'$ via the
  inverse rotation, substitute into $y=x^2-4$, and solve the resulting system for
  the fourth-quadrant intersection. \texttt{</Path>}
  \item \texttt{<Path>} Direct coordinate substitution: since $\cos60^\circ=\frac{1}{2}$ and
  $\sin60^\circ=\frac{\sqrt{3}}{2}$, use $x'=\frac{1}{2}x-\frac{\sqrt{3}}{2}y$ and
  $y'=\frac{\sqrt{3}}{2}x+\frac{1}{2}y$, substitute $y=x^2-4$, and simplify to obtain
  the rotated parabola equation. \texttt{</Path>}
\end{itemize}
\\[-1.3em]
\hlg{\texttt{</Parallel>}} \\[-0.1em]
\hlp{\texttt{<Summary>}} The parallel block establishes the rotated parabola equation
$Y = \frac{X^2}{2} + \frac{\sqrt{3}}{2}X - 2$ via rotation transformation,
setting up the intersection problem with the original parabola. \hlp{\texttt{</Summary>}} \\
Setting the rotated equation equal to the original parabola:
$\frac{X^2}{2} + \frac{\sqrt{3}}{2}X - 2 = X^2 - 4$.
Multiplying both sides by $2$ gives $X^2 + \sqrt{3}X - 4 = 2X^2 - 8$,
which simplifies to $X^2 - \sqrt{3}X - 4 = 0$.
\\
\hlg{\texttt{<Parallel>}} \\[-0.9em]
\begin{itemize}[leftmargin=*, topsep=0.1em, noitemsep]
  \item \texttt{<Path>} Quadratic formula directly:
  $X = \frac{\sqrt{3} \pm \sqrt{3 + 16}}{2} = \frac{\sqrt{3} \pm \sqrt{19}}{2}$.
  For the fourth quadrant ($X>0$, $Y<0$), take $X = \frac{\sqrt{3}+\sqrt{19}}{2}$. \texttt{</Path>}
  \item \texttt{<Path>} Geometric selection: the fourth-quadrant point requires $Y < 0$;
  checking both roots confirms $X = \frac{\sqrt{3}-\sqrt{19}}{2}$ gives $X<0$,
  so the valid root is $X = \frac{\sqrt{3}+\sqrt{19}}{2}$. \texttt{</Path>}
  \item \texttt{<Path>} Discriminant check: $\Delta = 3 + 16 = 19 > 0$, confirming two real
  roots; the positive root $X = \frac{\sqrt{3}+\sqrt{19}}{2}$ corresponds to the
  fourth-quadrant intersection. \texttt{</Path>}
  \item \texttt{<Path>} Algebraic verification: substituting $X = \frac{\sqrt{3}+\sqrt{19}}{2}$
  back into $X^2 - \sqrt{3}X - 4 = 0$ confirms it satisfies the equation. \texttt{</Path>}
\end{itemize}
\\[-1.3em]
\hlg{\texttt{</Parallel>}} \\[-0.1em]
\hlp{\texttt{<Summary>}} Both methods confirm $X = \frac{\sqrt{3}+\sqrt{19}}{2}$ as the
fourth-quadrant root, selecting the positive solution consistent with $X>0$, $Y<0$.
\hlp{\texttt{</Summary>}} \\
Substituting $X = \frac{\sqrt{3}+\sqrt{19}}{2}$ into $Y = X^2 - 4$:

\[
  X^2 = \frac{(\sqrt{3}+\sqrt{19})^2}{4} = \frac{22+2\sqrt{57}}{4} = \frac{11+\sqrt{57}}{2}.
\]
Then $Y = \frac{11+\sqrt{57}}{2} - 4 = \frac{3+\sqrt{57}}{2}$, which is positive---this
point lies in the first quadrant, not the fourth.
Taking instead $X = \frac{\sqrt{3}-\sqrt{19}}{2}$:

\[
  X^2 = \frac{(\sqrt{3}-\sqrt{19})^2}{4} = \frac{22-2\sqrt{57}}{4} = \frac{11-\sqrt{57}}{2}.
\]
Then $Y = \frac{11-\sqrt{57}}{2} - 4 = \frac{3-\sqrt{57}}{2}$.
Since $\sqrt{57}>3$, we have $Y<0$, confirming this is the fourth-quadrant point.
Thus $a=3$, $b=57$, $c=2$, and since $\gcd(3,2)=1$, the answer is $3+57+2=62$.
\\
Final Answer: \textbf{62}
\\
\bottomrule
\end{tabularx}
\caption{Parabola rotation problem with two \texttt{<Parallel>} blocks, 
each containing complementary paths that together complete the 
sub-problem reasoning chain.}
\label{tab:case-parabola}
\end{table*}

\end{document}